\documentclass{article}

\usepackage[accepted]{icml2024}

\usepackage{amsmath,amsfonts,bm}

\def\eqref#1{equation~\ref{#1}}

\def\1{\bm{1}}

\DeclareMathAlphabet{\mathsfit}{\encodingdefault}{\sfdefault}{m}{sl}
\SetMathAlphabet{\mathsfit}{bold}{\encodingdefault}{\sfdefault}{bx}{n}

\usepackage[utf8]{inputenc} %
\usepackage[T1]{fontenc}    %
\usepackage{url}            %
\usepackage{booktabs}       %
\usepackage{amsfonts}       %
\usepackage{nicefrac}       %
\usepackage{microtype}      %

\usepackage[normalem]{ulem}

\usepackage{float}
\usepackage{graphicx}
\usepackage[export]{adjustbox}
\usepackage[inline]{enumitem}

\usepackage{amsmath,amssymb,amsthm}
\usepackage[mathscr]{eucal}
\usepackage{stmaryrd}

\usepackage{setspace}

\usepackage{tabularx} %

\usepackage{booktabs} %

\usepackage{algorithm,algorithmic}
\renewcommand{\algorithmicrequire}{\textbf{Input:}}
\renewcommand{\algorithmicensure}{\textbf{Output:}}

\usepackage{soul}
\usepackage[textwidth=2cm,textsize=tiny]{todonotes}

\newcommand\numberthis{\addtocounter{equation}{1}\tag{\theequation}}

\graphicspath{{figs//}}

\allowdisplaybreaks

\renewcommand{\Pr}{\mathbb{P}}

\newtheorem{lemma}{Lemma}
\newtheorem{definition}[lemma]{Definition}

\newtheorem{proposition}[lemma]{Proposition}
\newtheorem{fact}{Fact}

\newtheorem{remark}{Remark}

\newtheorem{failure}{Failure}
\newtheorem{subfailure}{Failure}

\theoremstyle{definition}

\newcommand{\vocab}{\mathcal{V}}

\newcommand{\seq}{\bm{s}}

\newcommand{\seqtoken}{s}

\newcommand{\prefixlen}{L_{\texttt{pref}}}

\newcommand{\prefix}{\bm{p}}
\newcommand{\prefixtoken}{p}

\newcommand{\responselen}{L_{\texttt{resp}}}
\newcommand{\response}{\bm{r}}
\newcommand{\responsetoken}{r}

\newcommand{\distr}{\mathcal{D}}

\newcommand{\goal}{{v}_{\tt{goal}}}
\newcommand{\start}{{v}_{\tt{start}}}
\newcommand{\node}{v}
\newcommand{\asseq}{\texttt{adj}}
\renewcommand{\path}{\bm{r}}

\newcommand{\llm}{\texttt{LM}_{\theta}}
\newcommand{\llmprob}{{\texttt {LM}}_{\theta}}
\newcommand{\llmnoparam}{\texttt{LM}}

\newcommand{\trueprob}{{\Pr}_{\distr}} %
\newcommand{\ntpobj}{\mathcal{J}_{\texttt{next-token}}}

\newcommand{\foresightobj}{\mathcal{J}_{\texttt{t-less}}}

\usepackage{mathtools}
\newcommand\agsample{\stackrel{\mathclap{\normalfont\mbox{{\tiny \tt ag}}}}{\sim}}
\newcommand\foresightsample{\stackrel{\mathclap{\normalfont\mbox{{\tiny \tt \$}}}}{\sim}}

\newcommand{\agacc}{\texttt{Acc}_{\texttt{ag}}}
\newcommand{\cheatacc}{\texttt{Acc}_{\texttt{cheat}}}

\newcommand{\reverseacc}{\texttt{Acc}_{\texttt{rev}}}

\newcommand{\dollaracc}{\texttt{Acc}_{\texttt{\$}}}

\newcommand{\firstacc}{\texttt{Acc}_{\texttt{1st}}}

\newcommand{\snowballacc}
{\texttt{Acc}_{\texttt{sb}}}

\newcommand{\outpath}{\texttt{path}}

\usepackage{xcolor}   
\usepackage{dsfont}
\usepackage{listings}
\usepackage{float}
\usepackage{multirow}
\usepackage{booktabs}
\usepackage[textwidth=2cm]{todonotes}
\usepackage{hyperref}
\usepackage{url}
\newif\ifdraft
\draftfalse

\usepackage{microtype}
\usepackage{graphicx}
\usepackage{subfigure}
\usepackage{booktabs} %

\usepackage{hyperref}
\hypersetup{
    urlcolor=blue
    }

\usepackage{amsmath}
\usepackage{amssymb}
\usepackage{mathtools}
\usepackage{amsthm}

\usepackage[capitalize,noabbrev]{cleveref}

\usepackage[textsize=tiny]{todonotes}

\icmltitlerunning{The Pitfalls of Next-Token Prediction}

\begin{document}

\twocolumn[
\icmltitle{The Pitfalls of Next-Token Prediction}

\icmlsetsymbol{equal}{*}

\begin{icmlauthorlist}
\icmlauthor{Gregor Bachmann}{equal,yyy}
\icmlauthor{Vaishnavh Nagarajan}{equal,goog}
\end{icmlauthorlist}

\icmlaffiliation{yyy}{ETH Zürich, Switzerland}
\icmlaffiliation{goog}{Google Research, US}

\icmlcorrespondingauthor{Gregor Bachmann}{gregorb@ethz.ch}
\icmlcorrespondingauthor{Vaishnavh Nagarajan}{vaishnavh@google.com}

\icmlkeywords{Machine Learning, ICML}

\vskip 0.3in
]

\printAffiliationsAndNotice{\icmlEqualContribution} %

\begin{abstract}
Can a mere next-token predictor faithfully model human intelligence? %
We crystallize this emerging concern
{and correct popular misconceptions surrounding it, and advocate a simple multi-token objective}. As a starting point, we argue that the two often-conflated phases of next-token prediction --- autoregressive inference  and teacher-forced training --- must be treated distinctly.
The popular criticism that errors can compound during autoregressive inference, crucially assumes that teacher-forcing has learned an accurate next-token predictor. This assumption sidesteps a more deep-rooted problem we expose: in certain classes of tasks, teacher-forcing can simply fail to learn an accurate next-token predictor in the first place. %
We describe a general mechanism of how teacher-forcing can fail, and design a minimal planning task where both the Transformer and the Mamba architecture empirically fail in that manner --- remarkably, despite the task being straightforward to learn. 
Finally, we provide preliminary evidence that this failure can be resolved using \textit{teacherless} training, a simple modification using dummy tokens that predicts multiple tokens in advance. 
We hope this finding can ground future debates and inspire explorations beyond the next-token prediction paradigm. We make our code available under 
\url{https://github.com/gregorbachmann/Next-Token-Failures}

\end{abstract}

\section{Introduction}
Long after its inception in the seminal work of \citet{shannon48ntp,shannon51english}, next-token prediction has made its way into the core of the modern language model. But despite its long list of achievements, there is a small but growing belief that a next-token predicting model is merely an impressive \textit{improv} artist that cannot truly model human thought. 
Humans, when navigating the world,
 meticulously imagine, curate and backtrack plans in their heads before executing them. Such strategies are unfortunately not explicitly built into the backbone of the present-day language model. This criticism has been floating around as an informal viewpoint  \citep{lecun24ar,bubeck23sparks}. Our paper is aimed at crystallizing this intuitive criticism, clarifying popular misconceptions, and developing new core arguments for the next-token prediction debate.  %

Let us start by making more precise, what it means to say that human-generated language, or problem-solving, does not follow next-token prediction. When formalizing this, we hit an immediate roadblock: isn’t every sequence generation task possible autoregressively? Put differently, an optimist would say, every distribution over a sequence of tokens can be captured by an appropriately sophisticated next-token predictor simulating the chain rule of probability i.e., $\Pr(\responsetoken_1, \responsetoken_2, \hdots ) = \prod_{i} \Pr(\responsetoken_i \vert \responsetoken_1 \hdots \responsetoken_{i-1}) $. Thus, the autoregressivity in our systems is not antithetical to learning human language, after all.

Although this argument is compelling, a pessimist would worry, realistically, even with minor imperfections in the next-token predictor, the accuracy may fall spectacularly for long sequences \citep{kaariainen2006lower,ross10efficient,lecun24ar,dziri23faith}. Say, even if every next-token error is as little as $0.01$, the probability of encountering an erroneous token exponentially compounds along the way, and by the end of $200$ tokens, blows up to $0.86$. %

This is a simple and powerful observation. Yet, we explain why this does not completely capture the intuition that next-token predictors may be poor planners. Crucially, this argument does not carefully distinguish between the two types of next-token prediction: inference-time autoregression (where the model consumes its own previous outputs as inputs), and training-time {\em teacher-forcing} \citep{williams89tf}  (where the model is taught to predict token-by-token consuming all previous ground truth tokens as inputs). Framed this way, the  compounding of errors only pinpoints a superficial failure to execute a plan during \textit{inference}. %
It leaves open the possibility that we may have still learned a near-perfect next-token predictor; perhaps, with an appropriate post-hoc wrapper that verifies and backtracks, we can elicit the right plan without compounding errors.

Drawing this distinction allows us to articulate a much more concerning possibility: is it safe to assume that next-token based  learning (teacher-forcing) always learns an accurate next-token predictor? We identify this is not always the case.  Consider a task where we expect the model to witness a problem statement $\prefix = (\prefixtoken_1, \prefixtoken_2 \hdots, )$ and produce the ground truth response tokens $(\responsetoken_1, \responsetoken_2, \hdots)$.  Teacher-forcing 
trains the model to produce each token $\responsetoken_i$ by not only providing the problem statement $\prefix$ but also by revealing part of the ground truth $\responsetoken_1, \hdots \responsetoken_{i-1}$. Depending on the task, we first argue that this can induce shortcuts that use the revealed prefix of the ground truth answer to spuriously fit future answer tokens. We call this the \textit{Clever Hans cheat}. \footnote{{\em Clever Hans} \citep{bhlitem116908} was a famous show horse that could solve simple arithmetic tasks %
    by repeatedly tapping with his hoof until he reached the correct count. It turns out however, Clever Hans did not really solve the 
    problem, but merely
    stopped tapping upon detecting certain (involuntary) cues from his coach. %
    Clever Hans' answers were wrong when the coach was absent.}
Next, while the later tokens ($\responsetoken_i$ for large $i$) become easy to fit by the Clever Hans cheat, in contrast, the earlier answer tokens (say, $r_0, r_1$ etc.,) become harder to learn. This is because they no longer come with any supervision about the full answer --- part of the supervision is lost to the Clever Hans cheat. 
We argue that these two flaws would together arise in ``lookahead tasks'': tasks that require implicitly planning a later token in advance of an earlier token. In such tasks, teacher-forcing would result in a highly inaccurate next-token predictor that would struggle to generalize to unseen problems $\prefix$, even those sampled in-distribution.

Empirically, we demonstrate that the above mechanism leads to complete in-distribution failure in a path-finding setup on a graph, that we propose as a minimal lookahead task. We design our setup in a way that it is demonstrably straightforward to solve, implying that the failure of any model is remarkable. Yet, we observe  failure for both the 
Transformer \citep{vaswani17attention} and the Mamba architecture, a structured state space model \citep{gu2023mamba}. We then point towards
a very simple multi-token modification called  \textit{teacherless} training  ---  an idea that has appeared in other contexts \citep{tschannen23parallel,monea23pass,zhao23act} --- which predicts multiple future tokens and is able to circumvent
this failure in some settings.  Thus, we pinpoint a precise and easy-to-learn scenario where, rather than properties that are criticized in existing literature --- like convolution or recurrence or autoregressive inference (see \S\ref{sec:related}), --- it is next-token prediction during training that is at fault.

 We hope that these findings inspire and set future debates around next-token prediction on solid ground. In particular, we believe that the failure of the next-token prediction objective on our straightforward task casts a shadow over its promise on more complex tasks (such as say, learning to write stories). We also hope that this minimal example of failure and the positive results on teacherless training can motivate alternative paradigms of training.

We summarize our contributions below. 

\begin{enumerate}[leftmargin=1.25em,itemsep=0mm]
    \item We consolidate existing critiques against next-token prediction and crystallize new core points of contention (\S\ref{sec:related} and \S\ref{sec:intro-args}, \S\ref{sec:cleverhans}). 

    \item We identify that the next-token prediction debate must not conflate autoregressive inference with teacher-forcing. Both lead to vastly different failures (\S\ref{sec:intro-args},\S\ref{sec:diff}).

    \item We conceptually argue that in lookahead tasks, next-token prediction during training (i.e., teacher-forcing) can give rise to problematic learning mechanisms that are detrimental to even in-distribution performance (\S\ref{sec:cleverhans}). 
    
    \item We design a minimal lookahead task (\S\ref{sec:minimal-task}). 
     We empirically demonstrate the failure of teacher-forcing for the Transformer and Mamba architectures, despite the task being easy to learn (\S\ref{sec:verify}). 

    \item We identify that a teacherless form of training that predicts multiple future tokens at once --- proposed in \citet{monea23pass} for orthogonal inference-time efficiency goals --- shows promise in circumventing these training-time failures in some settings (\S\ref{sec:verify}, Eq~\ref{eq:foresight-obj}). This further demonstrates the limits of next-token prediction.

\end{enumerate}

\section{The Two Modes of Next-Token Prediction}

\label{sec:ntp}

Consider a set of tokens $\vocab$. Let $\distr$ be a ground truth distribution over sequences  that consist of a prefix  $\prefix$ and a response $\response$, denoted as $\seq = \prefix, \response$ where $\prefix = (\prefixtoken_1, \prefixtoken_2, \hdots, ) \in \vocab^{\prefixlen}$ and $\response = (\responsetoken_1, \responsetoken_2, \hdots) \in \vocab^{\responselen}$. 
We assume sequences of fixed length merely for simplicity. %

For any sequence $\seq$, let $\seq_{< i}$ denote the first $i-1$ tokens of $\seq$, and $\seq_{i<}$ the tokens following the $i$th token. Note that $\seq_{< 1}$ is the empty prefix. 
With an abuse of notation, let $\trueprob({\seqtoken}_i | \seq_{< i})$ denote the ground truth probability mass on ${\seqtoken}_{i}$ being the $i$th token given the prefix $\seq_{< i}$.
Consider a next-token-predicting language model $\llm$ (with parameters $\theta$) such that $\llmprob(\hat{\seqtoken}_{i} = \seqtoken_i ; \seq_{< i})$ is the probability that the model assigns to the $i$th output $\hat{\seqtoken}_i$ taking the value $\seqtoken_i$, given as {input} the sequence $\seq_{< i}$. Note that the next-token predictor only defines the probability for a single future token given an input, but not the joint probability of multiple future tokens. This joint probability is axiomatically defined analagous to the chain rule of probability: 
\begin{equation}
    \llmprob(\hat{\response}=\response \; ;  \prefix) \coloneq \prod_{i=1}^{\responselen} \llmprob\left( \hat{\responsetoken}_i = {\responsetoken}_i ; \prefix, \response_{<i} \right)
\end{equation}
where $\hat{\response}=\response$ denotes an exact token-by-token match.

 To train the above model, %
two distinct types of next-token prediction are used. First, during inference, for a given prefix, we autoregressively sample from the model token-by-token, providing as input the prefix and all previously-generated tokens. Formally,

\begin{definition} \textbf{(Inference-time next-token prediction via autoregression)} Autoregressive inference is a form of inference-time next-token prediction in that to generate a response $\hat{\response}$, we iterate over $i=1, \hdots, \responselen$, to sample the next token $\hat{\responsetoken}_i$  with the distribution given by  $\llmprob(\hat{\responsetoken}_i \; ; \prefix, \hat{\response}_{<i})$. We denote this as $\hat{\response}  \agsample \llmprob(\cdot \; ; \prefix)$.   
\end{definition}

There is also a second phase of next-token prediction, one that is applied during the training process, called \textit{teacher-forcing}. Here, instead of feeding the model its own output back as input, the model is fed with prefixes of the \textit{ground truth} response $\response_{<i}$. Meanwhile, the model is assigned as supervisory target, $r_{i}$, the next ground truth token. Then, the model maximizes a sum of next-token log-probabilities:

\begin{definition}(\textbf{Training-time next-token prediction via teacher-forcing}) Teacher-forced training is a form of training-time next-token prediction in that we find parameters $\theta$ that maximize the next-token log-probability sum:
    \begin{align*}
    \ntpobj(\theta) &= \mathbb{E}_{(\prefix,\response) \sim \distr}\left[ \log \llmprob \left( \hat{\response} = {\response} \; ; \prefix \right) \right] \\
    &= \mathbb{E}_{\distr}\Big[ \sum_{i=1}^{\responselen} \log \llmprob \left( \hat{\responsetoken}_i = {\responsetoken}_i ; \prefix, \response_{<i} \right) \Big]
\numberthis    \label{eq:ntp-obj}
\end{align*}
\end{definition}

The key property of the objective is that we extract the model's output, {\em allowing the model access to the ground truth response preceding the current token}. This property will be crucial to the failure we describe in \S\ref{sec:cleverhans}.

 \section{Failure due to Auto-Regressive Inference}
\label{sec:intro-args}
A broad criticism against next-token predictors is that intuitively these models are not explicitly designed to plan ahead, and during inference, they do not know how to recover from their own errors. This discourse has been fragmented in literature. Furthermore, the umbrella term ``next-token prediction'' is used interchangeably with ``autoregressive architecture''. %
Our goal is to analyze these intuitions more systematically, and be careful about distinguishing between the two phases of next-token prediction: teacher-forcing and autoregression.
A key insight we will arrive at is that existing arguments capture only a part of the intuitive concern that next-token predictors may not be able to plan.

\textbf{The chain-rule-of-probability defense:} We first outline what is arguably the most tempting defense for next-token prediction: the chain rule of probability always promises us a next-token predictor that can fit our distribution.

\begin{fact} \textbf{\textit{(Every sequence distribution can be represented by a next-token predictor)}} By the chain rule of probability we have $\trueprob(\response \; | \; \prefix) = \prod_{i=1}^{\responselen} \trueprob(\responsetoken_i \; | \; \prefix, \response_{< i})$. Therefore, define a next-token predictor $\llmnoparam$ such that for every valid value of $i, \prefix$, and $\response$, we have $\llmnoparam (\hat{\responsetoken}_i = \responsetoken_i \; ; \prefix, \response_{< i} ) \coloneq \trueprob(\responsetoken_i | \prefix, \response{< i})$.  Then, sampling $\response \sim \distr | \prefix$, is equivalent to autoregressively sampling $\response \agsample \llmnoparam(\cdot \; ; \prefix) $.
\label{fact:crp}
\end{fact}
The cleverness of this argument lies in the fact that it can apply to \textit{any} imaginable distribution.  Thus, as long as the next-token predictor is sufficiently expressive (by scaling up the context, memory and compute), it can model both natural language and problem-solving. Thus, it may seem that next-token predictors are not antithetical to planning-based tasks, after all.

\textbf{The snowballing errors criticism}:
A skeptic would however raise the following concern. Regardless of the abundance of computational resources, realistic models %
may still predict the next token with a slight probability of error in each step; 
these error probabilities may then exponentially accumulate over time.
This has been formalized in various contexts, from that of autoregressive models \citet{lecun24ar}, to that of the limits of Transformers in compositional tasks \citet{dziri23faith}, and in a different form, in much earlier work in imitation learning and structured prediction \citet{kaariainen2006lower,ross10efficient} (see \S\ref{sec:related}). 
We present a minimal formalization of this below: %

\begin{failure} \textbf{(Snowballing error due to autoregressive inference)}
Consider a model $\llm$, prefix $\prefix$ and a unique ground truth response $\response$ such that the next-token error obeys 
\begin{equation} \forall i \leq \responselen, \; \llmprob \left( \hat{\responsetoken}_i \neq {\responsetoken}_i  ; \prefix, \response_{<i} \right) \approx \epsilon. \label{eq:snowball-condition} \end{equation}
Then, for $\hat{\response} \agsample \llm(\cdot \; ; \prefix) $ the probability that the generated response exactly matches the ground truth  $\response$ obeys
\[ 
\mathbb{P}(\hat{\response} = \response) 
 \approx (1-\epsilon)^{\responselen}. %
\]
\label{fail:snowball}
\end{failure} 
\vspace{-2em}
We argue that the snowball failure mode only indicates how an autoregressive model can fail to \textit{execute} a plan during inference-time. It does not  preclude the possibility that the model may have \textit{learned} a good plan that it simply fails to execute during inference. Concretely, it may still be possible that, at each step, the model has high accuracy of predicting a  next token that is consistent with a good plan (as assumed in Eq~\ref{eq:snowball-condition}). Depending on the setting, one can potentially exploit this accuracy to elicit a good plan during inference. For instance, one may be able to use a post-hoc wrapper that verifies whether an error has taken place, then backtracks and executes a different action. One may even simulate backtracking using more elaborate techniques such as tree-of-thought \citep{wei22cot,yao23tree,besta23graph,yao23react}, or using the model to give itself feedback \citep{madan23selfrefine,huang22internal,shinn2023reflexion} to elicit the plan that the model has learned.

Thus, the snowball failure mode captures what is primarily a shortcoming of an autoregressive architecture. Likewise, the chain-rule-of-probability defense captures only the expressive power of an autoregressive architecture. Neither of these arguments address the possibility that learning with next-token prediction may itself have shortcomings in learning how to plan. In this sense, we argue that existing arguments capture only a part of the intuitive concern that next-token predictors fare poorly at planning.

\section{Failure due to Teacher-Forcing}
\label{sec:cleverhans}

Can a model trained to predict the next token, fail to predict the next token with high accuracy during test-time? Mathematically, this would mean showing that a model trained with the teacher-forcing objective of Eq~\ref{eq:ntp-obj} has high next-token prediction error on the very distribution it was trained on (thus breaking the assumption in Eq~\ref{eq:snowball-condition} of the snowballing failure mode). Consequently, no post-hoc wrapper can salvage a plan out of the model. The goal of this section is to conceptually argue that this failure can happen for lookahead tasks: tasks that implicitly require computing a future token in advance before an earlier token. %

As a running example for our argument, we design a path-finding problem on a simple class of graphs. We view this example as a minimal setting that captures the core essence of what it means to solve problems with lookahead, %
without irrelevant confounding factors. This task is also demonstrably straightforward to solve, as we will see, thus making any observed failures remarkable. Thus we view this running example as a template for an intuitive argument that can be made about teacher-forced models on more general and harder problems that require lookahead.

\begin{figure}
    \centering
\includegraphics[width=0.48\textwidth]{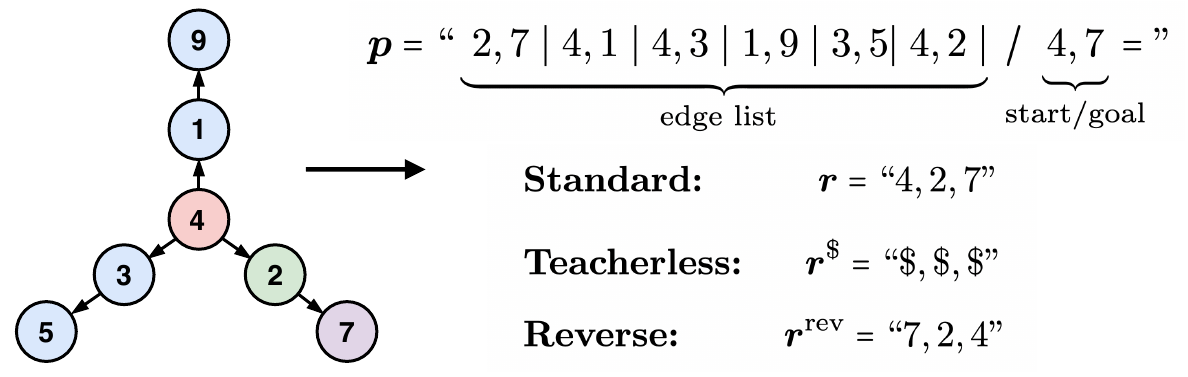}
    \caption{Illustration of a path-star graph. The prefix $\prefix$ represents the adjacency list and the (central) start and goal node. The target is represented by $\response$. Under ``standard'' teacher-forcing, we condition the model on prefixes of $\response$ to predict $\response$. But in \S\ref{sec:verify} we explore alternatives where we train without a teacher (condition on $\response^{\$}$ and predict $\response$) or train with a reversal (condition on and predict $\response^{\texttt{rev}}$).}
    \label{fig:graph-seq} 
\end{figure}

\subsection{Path-Finding on Path-Star Graphs: A Minimal and Easy Lookahead Task}
\label{sec:minimal-task}
Consider a path-finding problem on a directed graph $G$ with a set of nodes $\{\start, \goal, \node_1, \node_2, \hdots \}$.  The graph is a ``path-star'' graph with $\start$ as the central node, with at least $2$ paths (each of length $l \geq 2$ edges) emanating from it, with a unique path ending in $\goal$. The task is to find a path from $\start$ to $\goal$. Correspondingly, we assume that the distribution $\distr$ is over sequences where the prefix $\prefix$ represents a (randomly generated) graph, and the response represents the path from the start to the goal. In particular, we sample a graph $G$ which is represented as an adjacency list as $\asseq(G) = e_1, e_2, \hdots$ where each edge $e=(\node, \node')$ is represented such that $\node'$ farther away from $\start$ than $\node$. We then set the prefix as $\prefix = (\asseq(G), \start, \goal)$ so the model knows what the graph, and the desired start and goal states are. The ground truth response $\response$ corresponds to the sequence of vertices $\response = \start, \hdots \goal$ on the start-to-goal path.
We visualize this construction in Fig.~\ref{fig:graph-seq}.

\textbf{The straightforward lookahead solution.} Ideally, we want the model to learn a mapping from the input $\prefix$ consisting \textit{only} of $(\asseq(G),\start, \goal)$ to an output that is the full path $\path$. %
Two such solutions are possible. One idea is to {plan} by examining all the paths emanating from $\start$ and choosing the one that ends at $\goal$. But a second, straightforward solution exists: the model simply needs to {look ahead} at the sequence ``right-to-left'' and observe that it corresponds to the one unique path starting from $\goal$ and ending at $\start$. After internally computing the path from $\goal$ and reversing it, the model can emit its response.

\subsection{Outline of Failure Mechanism}

While we will use the path-star example as a running example, we make our claim more generally for problems that require lookahead (such as story-writing, as we will discuss later). %
 In such problems, we claim that teacher-forcing prevents learning the true mechanisms, causing failure. Intuitively, in teacher-forcing, we decompose the learning of $\prefix \to \response$ into multiple problems, one for each token $\responsetoken_i$. Specifically, we make the model learn a mapping from the input $(\prefix, \path_{< i})$ --- {not just}  $\prefix$ --- to the output $\responsetoken_i$.
The additional information $ \path_{< i}$ in the input, we argue, is problematic and destroys the core challenge in what the model has to learn. 
Specifically, our argument puts forth two debilitating mechanisms that would together emerge under teacher-forcing  (explained over the next two subsections). While, we will empirically verify these mechanisms for path-star graphs in \S\ref{sec:verify}, we also provide a discussion of how our ideas apply to a text-based scenario at the end of this section.

\begin{figure}
    \centering
    \includegraphics[width=0.5\textwidth]{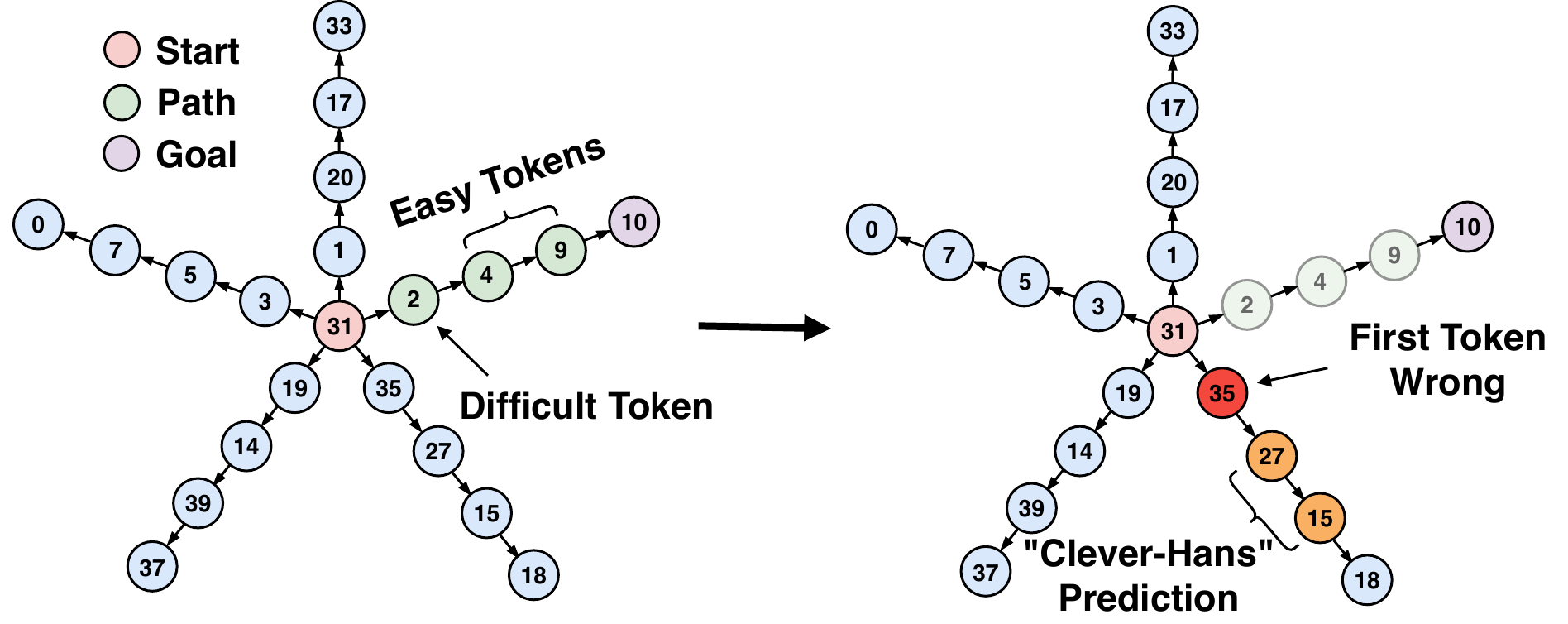}
    \caption{Illustration of the failure of teacher-forcing on a path-star graph. The left image marks the ``easy tokens'' which can be fit by the Clever Hans cheat (Failure~\ref{fail:clever}), while the ``difficult token'' cannot be learned (Failure~\ref{fail:diff-token}) due to lost supervision. The right image shows how the model would behave during autoregressive inference, under the absence of the ``teacher''.}
    \label{fig:star-graph}
\end{figure}

\subsection{The Clever Hans Cheat}

First, and most importantly, by revealing parts of the answer to the model as input, we allow the model to fit the data by \textit{cheating} i.e., by using trivial mechanisms that use the extra information in $\response_{<i}$ to produce $\responsetoken_i$. Such cheats must especially be abundant for the later tokens (large $i$) for which a larger prefix is revealed.

To illustrate this in our path-star example, without loss of generality, consider a ground truth path that is of the form $\response = \start, \node_1, \node_2, \hdots, \goal$. With a slight abuse of the indexing notation, let $\response_{< i} = \start, \node_1, \hdots, \node_{i-1}$ be the prefix of length $i$ (so we index from $0$ instead of $1$). Observe that nodes from $\node_2$ onwards, until before $\goal$, have precisely one edge  going ``away'' from $\start$. 
Thus, consider when the model is given as input, $(\prefix, \response_{<i})$ where $\prefix = (\asseq(G), \start, \goal)$, to fit the target $\node_{i}$. The model first merely needs to  scan the adjacency list $\asseq(G)$ within $\prefix$ for the one edge containing $\node_{i-1}$ in the first position. Then, the model only has to predict the other node on that edge as $\node_{i}$.  Note though, this cheat  cannot work on fitting the target $\node_1$ given the input $\response_{<1} = \start$ since $\start$ has many outward edges --- we will scrutinize this node in the next section. We illustrate this difference between $\node_1$ and the remaining tokens as the ``easy'' vs. ``difficult'' tokens in Fig.~\ref{fig:star-graph}. 

Crucially, the above cheating mechanism for fitting the easy tokens does not require any lookahead. It is simple, and implementable by an induction head-like module \citep{olsson22icl}. Then, given the well-known simplicity bias of neural networks \citep{shah20simplicity}, we hypothesize that
the later tokens will be quickly fit and ignored during training. This destroys any gradient signal (a.k.a gradient starvation \citep{pezeshki21starvation}) to efficiently learn the ``right-to-left'' solution --- the solution that requires looking at all the tokens in $\response$, and learning that they are simply the unique path from $\goal$ spelled in reverse.

We emphasize two key aspects of this cheating behavior. First, these shortcuts are unlike well-known shortcuts (see Remark~\ref{rem:shortcuts}) that map
 from the original input prefixes $\prefix$  to the ground truth $\response$. 
 What we identify is unique to the mapping from the teacher-forced prefix $(\prefix, \response_{<i})$ to $\responsetoken_i$. Hence, 
 we name this {\em Clever Hans cheating}.
Another notable point is that this does not come from a dearth of samples: even if we had infinite training data at our disposal, the model can still fit the easy tokens of all that data by Clever Hans cheating.

\subsection{The Indecipherable Token}

Perhaps, not all is lost. While the later tokens may be fit using the Clever Hans cheat, we may still have some of the earlier tokens (for small $i$), for which such cheats may be unavailable. The supervision from these tokens may eventually coerce the model into learning the true solution.  For example, in the path-star task, the model still needs to learn to predict the first node $\node_{1}$, where it is not possible to fit the training data by the Clever Hans cheat. If not memorize this token on the data, the most general way to fit this token is by actually solving the underlying task.

However, we argue that it is significantly harder for the model to learn the correct solution now. Consider the moment in training when the Clever Hans cheat is perfected. At this point, the model is {deprived} of information about much of the full solution which was once present as supervisory targets. The model is simply left with the task of mapping the input $\prefix$ to an \textit{incomplete} solution (e.g., the first vertex $\node_1$ in the path-star graph). 
Recovering the plan in this scenario must first of all be relatively harder from a statistical point of view due to the incomplete supervision. But more importantly, learning this task may become computationally harder, or even intractable under certain assumptions. We provide an informal intuition using the path-star problem below, but this intuition should extend to more general problems as well --- indeed, \citet{weis23subtask} and other literature on chain-of-thought echo similar negative results about learning from limited supervision (see \S\ref{sec:more-related-work}).

Intuitively, our learner has to find an end-to-end algorithm that composes multiple subroutines. For instance, the straightfoward solution consists of $l$ steps: start from the current vertex as $\goal$, and find the preceding vertex in the graph in each subsequent step. Each vertex in this path can be thought of as ``intermediate supervision'' to learn a corresponding ``find-the-adjacent-vertex'' subroutine from a space of candidate subroutines.\footnote{As an illustration of what these candidate subroutines could be, imagine that the model can implement an induction head \citep{olsson22icl} $\texttt{Ind}_k(\prefix, \node)$ that finds $\node$ in the adjacency list of $\prefix$, and outputs the token that precedes it by $k$ positions. Then the candidate space could be parameterized by $k$ as $\{\texttt{Ind}_k(\prefix, \node) | k=1, 2, \hdots, \}$. For our specific tokenization, the correct subroutine at each of the $l$ steps is the induction head for which $k=2$.} Even if we conservatively assume that there is only a constant-sized space of candidate subroutines $\mathcal{C}$, the end-to-end search space is an exponential space of $|\mathcal{C}|^l$ algorithms composing $l$ subroutines. 

Now, after the Clever Hans cheat is in effect, the only supervision for this search is the single-token loss, $-\log \llmprob \left( \hat{\responsetoken}_1 = {\responsetoken}_1 ; \prefix \right)$. However, this loss is an ``all-or-nothing'' loss. Crucially, by the \textit{discrete} nature of the task, even if one subroutine is incorrect, the final answer $\hat{\responsetoken}_1$ would likely be incorrect on all inputs. For instance, imagine that the first subroutine is incorrect and its output takes us to an arbitrary location on the graph. Then, even if all subsequent subroutines were correct (i.e., they are ``find-the-adjacent-vertex'' subroutines), the final output would be arbitrary. Thus, we have $\hat{\responsetoken}_1 = {\responsetoken}_1$ precisely for the  algorithm where all $l$ subroutines are correct, and $\hat{\responsetoken}_1 \neq {\responsetoken}_1$ for any other choice of the algorithm. For such an all-or-nothing loss surface, the end-to-end learner must necessarily brute-force search the exponential space of algorithms. 
 We encapsulate the overall claim more generally below:

\begin{proposition}
    \label{prop:indecipherable}
    Let $\mathcal{C}$ be a set of \textit{discrete}-output candidate subroutines.
    Consider learning a task such that (i) it requires composing some $l$ subroutines from $\mathcal{C}$, (ii) the $k$ leading response tokens are sensitive in that even if one subroutine is altered, the first $k$ tokens are each completely altered. %
Then learning the task with only supervision from the first $k$ ground truth tokens requires exponential time of $\Omega(|\mathcal{C}|^l)$, given the assumptions listed in Remark~\ref{rem:indecipherable}. 
\end{proposition}

In \S\ref{sec:verify}, we will verify experiments to demonstrate that our models indeed fail to learn the Indecipherable Token as a result of Clever Hans cheating, and that conversely, they succeed whenever the Clever Hans cheat is prevented.

\subsection{Beyond the Path-Star Setting} 

Framing our argument more generally, and informally, we argue that teacher-forcing can suffer the following failures in order, especially in tasks that require advance lookahead.

\setcounter{failure}{2}
\begin{subfailure}(\textbf{Clever Hans cheating due to teacher-forcing}) Although there is a true mechanism that recovers each $\responsetoken_i$ from the original prefix $\prefix$, there may be multiple other mechanisms that can recover each token $\responsetoken_i$ from the teacher-forced prefix $(\prefix, \response_{< i})$. These mechanisms may be simpler thus disincentivizing the model from learning the true mechanism. 
\label{fail:clever}
\end{subfailure}

\begin{subfailure}(\textbf{Indecipherable token due to lost supervision}) After the Clever Hans cheat is perfected during training, the model is deprived of a part of the supervision (especially, $\responsetoken_i$ for larger $i$). This makes it harder and potentially even intractable to learn the true mechanism from the remaining tokens alone. 
\label{fail:diff-token}
\end{subfailure}

As we demonstrate in the next section, the above failures can cause the model to fail on the very distribution it was trained on. This breakdown of planning abilities emerges right from training, and is orthogonal to the Snowballing Failure that is primarily an inference-time issue (See \S\ref{sec:diff}).

While the path-star problem provides a concrete, verifiable setting of this failure, it can also help us speculate how such failures could occur in more complex and nebulous  tasks. More generally, we expect this failure to occur when there are right-to-left dependencies i.e., a later-appearing token must be planned before an earlier-appearing token. We provide an example below.

\textbf{Story-writing.} Imagine training on novels that take the form of a conflict, followed by a backstory, followed by a resolution of the conflict,  utilizing the backstory. %
Although the story explicitly reads as $\node_{\tt conflict}, \node_{\tt backstory}, \node_{\tt resolution}$, implicitly, one must learn to decide on  $\node_{\tt backstory}$ before all else.
We however conjecture that the teacher-forced model would suffer the Clever Hans cheat wherein it would first learn to fit $\node_{\tt resolution}$
using simple deductive skills. With this crucial part of the story lost as supervision, the 
model can no longer decipher how the remaining pieces relate i.e., how $\node_{\tt backstory}$ must be planned in advance of $\node_{\tt conflict}$. We conjecture that the resulting model would learn to generate uninteresting stories, interjecting arbitrary conflicts and backstories on a whim, subsequently forcing contrived resolutions upon them. %
While this hypothesis is not straightforward to empirically test for, we provide a more detailed conceptual illustration in \S\ref{sec:story}.

\section{Experimental Verification}
\label{sec:verify}
In this section, we demonstrate our hypothesized failure modes on the graph path-finding task.   We show this in both Transformers and Mamba to demonstrate that these failures are general to teacher-forced models. 
First, we establish that our teacher-forced models fit the training data but fail in-distribution. Next, we design metrics to quantify the extent to which the two hypothesized mechanisms (Failures~\ref{fail:clever},~\ref{fail:diff-token}) occur. Finally, we design alternative  objectives to intervene and remove each of the two failure modes, to test whether the performance improves. We report additional experiments in \S\ref{sec:arithmetic} for an arithmetic task, and in \S\ref{sec:snowball-expt} quantifying the Snowballing Failure~\ref{fail:snowball}.  We describe our experimental setting more precisely below.

\paragraph{Dataset.}%
We denote by ${G}_{d, l}(N)$ for $d, l, N \in \mathbb{N}$, a path-star graph consisting of a center node $\start$ with degree $d \in \mathbb{N}$, meaning there are $d$ different paths emerging from the center node, each consisting of $l-1$ nodes (excluding the start node). Node values are uniformly sampled from $(\{0, \dots, N-1\})$ where $N$ can be larger than the actual number of nodes in the path-star graph. %
In every graph, we use the center node as the starting node $\start$ and then pick as $\goal$, the last node of one of the paths chosen uniformly at random. %
The order of the edges in the adjacency list is randomized. We describe the tokenization in \S\ref{sec:tokenizer}.

For each experiment, we generate the training and test graphs from the same distribution $\distr$, all with the \textit{same} topology of ${G}_{d, l}(N)$ with fixed $d, l$ and $N$. Thus, any failure we demonstrate is an \textit{in-distribution} failure, and does not arise from the inability to generalize to different problem lengths \citep{anil22length}. While the graphs have identical topology, this is not a trivial memorization problem for the model, since the graphs are labeled differently, and the adjacency list randomized --- the model \textit{has} to learn a general algorithm. Throughout the experiments, we fix the number of samples to $200k$ and fix the number of node values to $N=100$ across topologies to enable diverse instantiations of the topology for training and testing.

\paragraph{Models.} We evaluate models from two architectural families to highlight that the failures are not tied to a particular architecture but stem from the next-token prediction objective. For Transformers, we use  from-scratch GPT-Mini, and pretrained  GPT-2 large \citep{Radford2019LanguageMA}. For recurrent models, we use from-scratch Mamba \citep{gu2023mamba}. We optimize using \textit{AdamW} \citep{loshchilov2018decoupled} until perfect training accuracy. To rule out grokking behaviour \citep{power2022grokking}, we trained the cheaper models for as long as $500$ epochs. More details are in \S\ref{sec:models}.

\subsection{Observations.}

\begin{figure*}
  \includegraphics[width=\textwidth,trim={0 0 0 10em},clip]{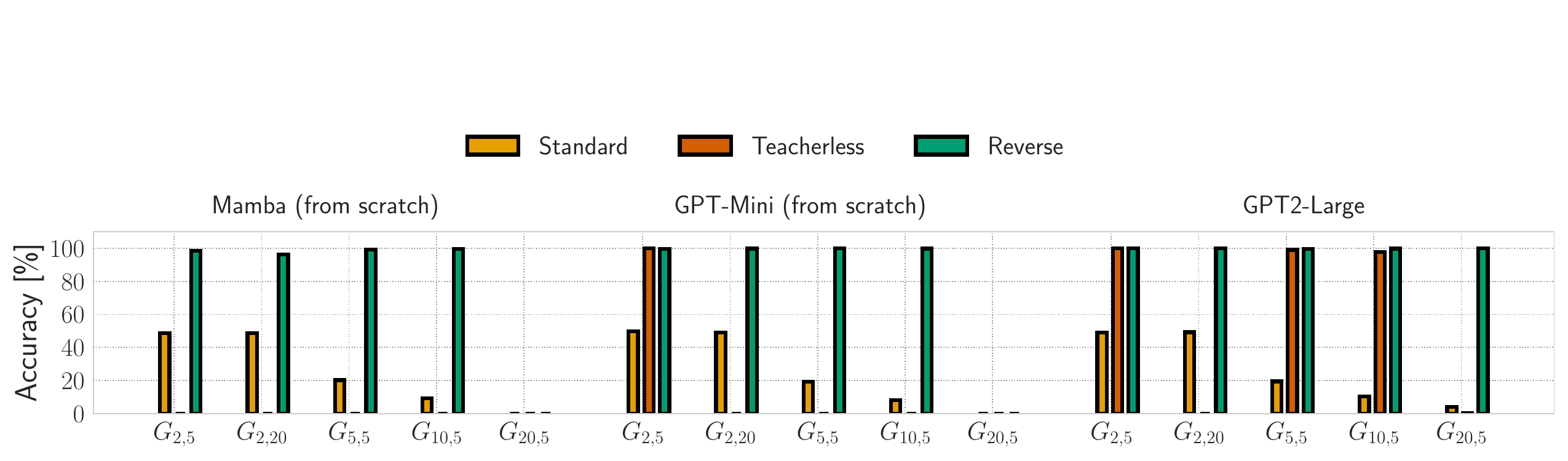}
  \caption{For different architectures, we report the accuracy of the standard teacher-forced model ($\agacc$, Eq~\ref{eq:gen-acc}), teacherless-trained model's accuracy ($\dollaracc$, Eq~\ref{eq:foresight-acc}) and accuracy of the model trained with reversed targets ($\reverseacc$, Eq~\ref{eq:rev-acc}) evaluated on path-finding a range of graphs (with degree in the first subscript, and path length in the second).}
  \label{fig:all-accs}
\end{figure*}

\begin{figure}
    \centering
\includegraphics[width=0.25\textwidth]{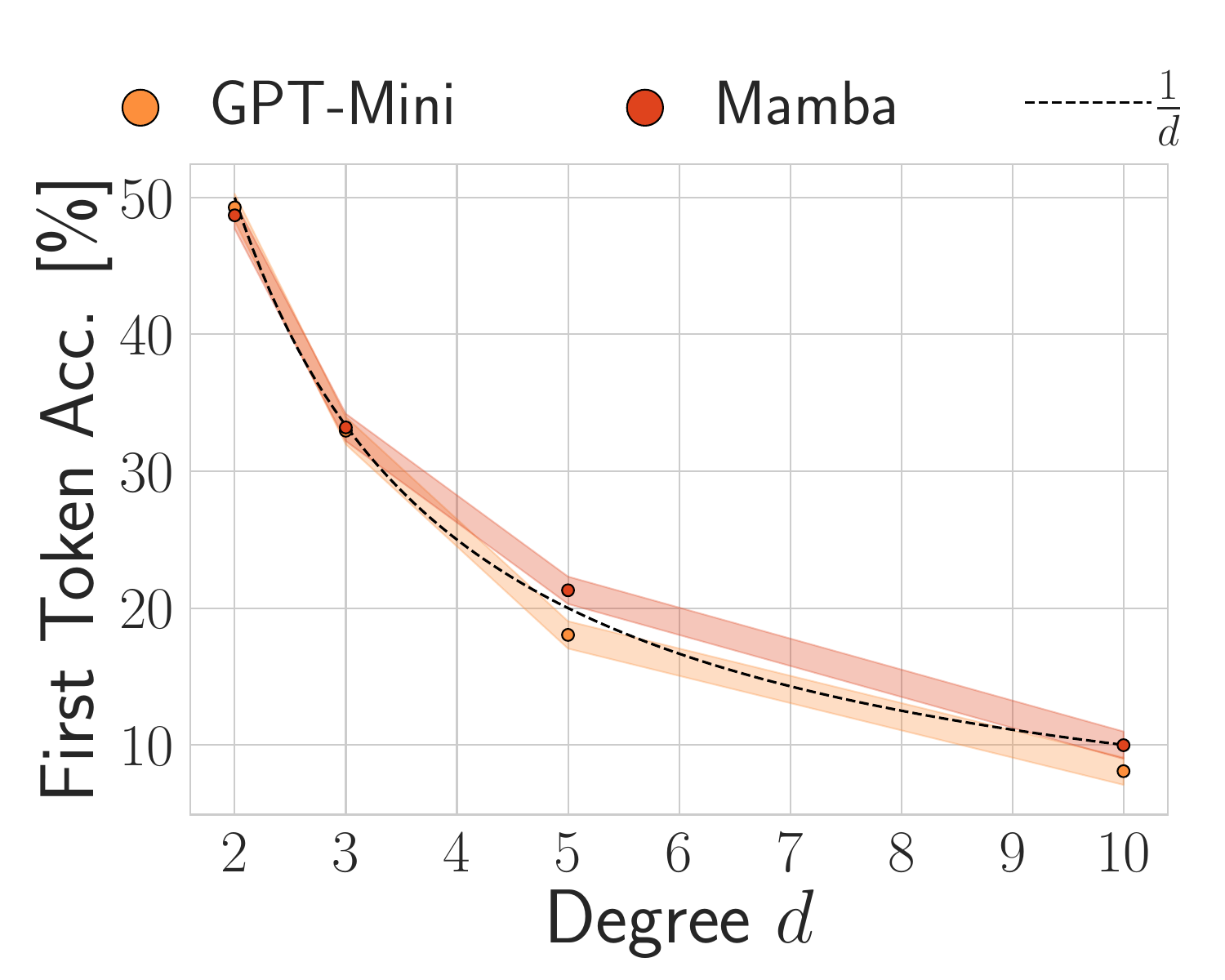}%
\includegraphics[width=0.25\textwidth]{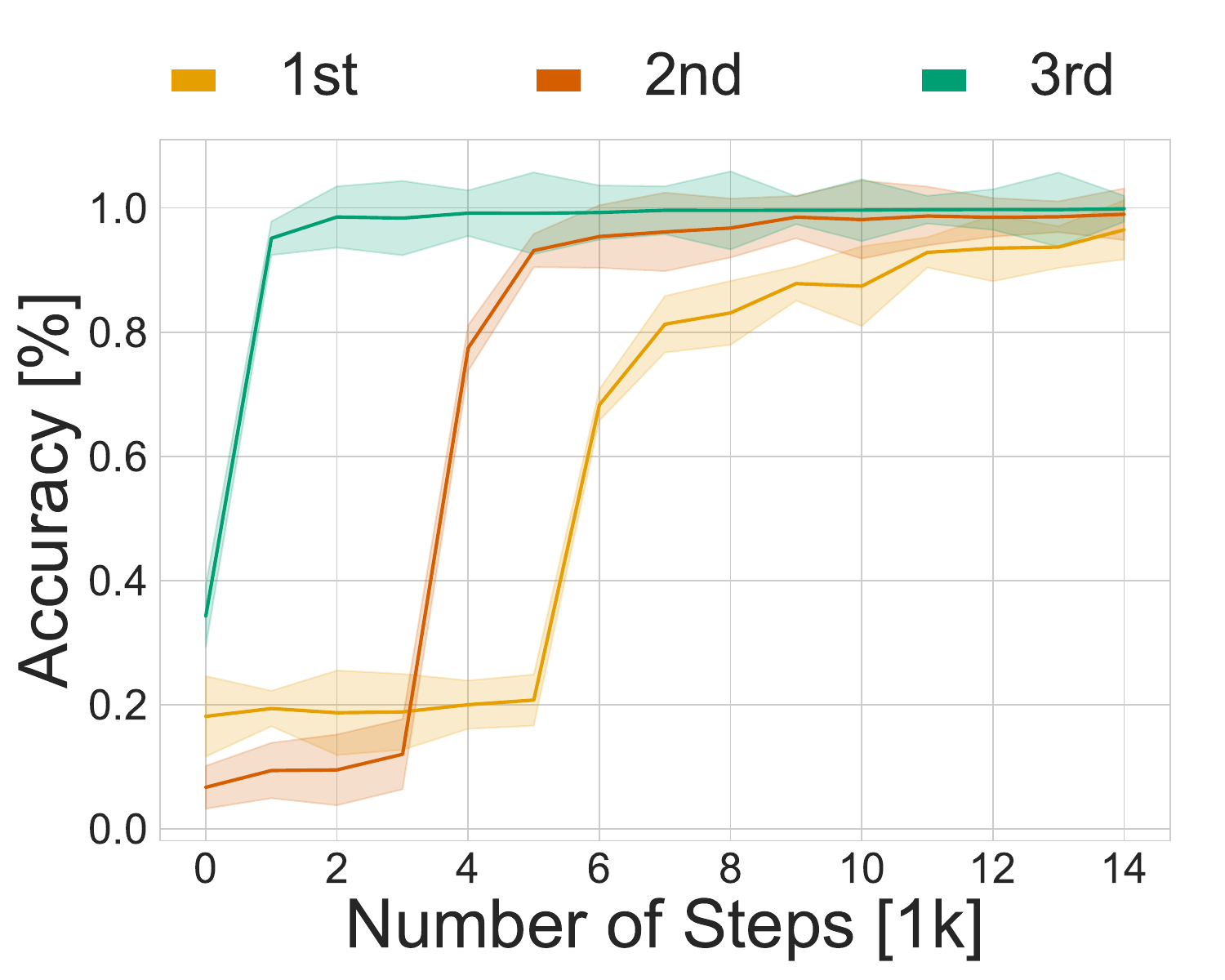}
    \caption{$\firstacc(\llm)$ (in percent $\%$, Eq~\ref{eq:first-acc}) for path-star graphs of various degrees $d \in \{2, 3, 5, 10\}$ for fixed path length $l=5$ (left). Individual token accuracies (for $\node_1, \node_2, \node_3$) for the graph $G_{5, 5}$ under teacherless training (Eq~\ref{eq:foresight-obj}) with GPT2-large  (right). }
    \label{fig:diff-traj}
\end{figure}

\paragraph{Verifying in-distribution failure.} 
For a given distribution, we evaluate all our teacher-forced models by autoregressively generating solutions, and comparing that solution with the true one for an exact-match. We denote this accuracy as $\agacc(\llm)$ (see Eq~\ref{eq:gen-acc}) and report it 
for path-star graphs of varying topologies in Fig.~\ref{fig:all-accs} and Table~\ref{table:gen-acc}. As observed, all models (even when pre-trained) struggle to learn the task accurately. The accuracies are precisely limited to the value when uniformly guessing a path from $\start$  i.e., $\approx \frac{1}{d}$, thus establishing complete in-distribution failure. This is so even when trained to fit sample sizes up to $200k$ to $100\%$ accuracy, and despite the fact that the training and test graphs have identical topology. Next, we quantitatively demonstrate how this stark failure arises from our two hypothesized mechanisms (Failure~\ref{fail:clever},~\ref{fail:diff-token}).

\paragraph{Verifying Failure \ref{fail:clever} (The Clever Hans cheat)} We had hypothesized that the teacher-forced model would cheat to fit the training tokens (the ones that follow $\responsetoken_1$ in each instance). Specifically, to predict node $\node_{i}$ in the true path, the model can exploit the ground truth node $\node_{i-1}$ that is revealed as input. Rather than learning to plan, the model would simply predict the node that is outwardly  adjacent to $\node_{i-1}$. To quantify whether this behavior emerges, we ``teacher-force'' the model with a uniform random neigbhor $\node_1'$ of $\start$. We then test whether the model indiscriminately applies the learned Clever Hans cheat here: does the model religiously follow the path that emanates from the neigbhor $\node_1'$, not necessarily ending in $\goal$? We measure the exact match of this path on a held-out set as $\cheatacc(\llm)$ ( Eq~\ref{eq:cheatacc}).

Empirically, in \S\ref{sec:more-results}, Table \ref{table:cheat-accs}, we find $\cheatacc(\llm) \approx 100\%$ almost across the board (except for high-degree graphs where training is challenging). This establishes that to fit the training data, the teacher-forced model has exploited the Clever Hans cheat.

\paragraph{Verifying Failure \ref{fail:diff-token} (The Indecipherable Token)} Recall that the Clever Hans cheat only applies to all but the first node $\node_1$ after $\start$ lying on the path.
After the Clever Hans cheat fits the rest of the path during training, we hypothesized that node $\node_1$ may become impossible to learn since the model is deprived of all information about the subsequent targets. To quantify this behavior, we evaluate how well the model is able to predict the difficult first node, $\node_1$. We measure this on the held-out set and denote this as $\firstacc(\llm)$ (see Eq~\ref{eq:first-acc}). 
As shown in Fig.~\ref{fig:diff-traj} the model achieves a low $\firstacc(\llm)$, approximately $1/d$. Thus, the model indeed fails to identify that $\node_1$ is the one on the path to $\goal$ and instead randomly emitting one of the $d$ neighbors of $\start$. %

\paragraph{Removing the Clever Hans cheat via teacherless training \citep{tschannen23parallel,monea23pass}} We now consider a training setup where we prevent Clever Hans cheating (Failure~\ref{fail:clever}) and examine how learning differs. 
Concretely, consider modifying teacher-forcing by replacing the \textit{input} $\response$ (which reveals the ground truth) with an uninformative input  $\response^{\$}$, consisting of the same special (``lookahead'') token $\$$ repeated $l$ times. %
For supervision in the loss, we still use the original target $\response$. Thus, the model cannot fit the targets by looking at the prefixes $\response_{<i}$ and by predicting the next token $\node_i$ via cheating. Instead, the model only has access to the graph description in $\prefix$ to lookahead and fit all the targets $\node_i$ for $i=1, \hdots, l$.  Formally, we maximize:
\begin{align*}
    \foresightobj(\theta) 
    &= \mathbb{E}_{\distr}\Big[ \sum_{i=1}^{\responselen} \log \llmprob \Big( \hat{\responsetoken}_i = {\responsetoken}_i ; \prefix, \response^{\$}_{<i} \Big) \Big]. \numberthis \label{eq:foresight-obj}
\end{align*}

We denote a model trained this way by $\llm^{\$}$ and perform inference simply by conditioning on 
$\$$ tokens i.e., to extract $\hat{\responsetoken}_i$,
we feed the uninformative prefix $\response^{\$}_{<i}$ as input, rather than autoregressively feeding the output $\hat{\response}_{<i}$ as input. We denote the resulting accuracy by  $\dollaracc(\llm^{\$})$ (see Eq~\ref{eq:foresight-acc}). 
Our goal however is to evaluate whether forcing the model to lookahead can dodge the Clever Hans cheat, thereby allowing the correct mechanism to be learned.

We report the accuracy of these teacherless models in Fig.~\ref{fig:all-accs} and Table~\ref{table:dollar-accs}. Unfortunately, in most cases, the teacherless objective is too hard for the models to even fit the training data, likely because there is no simple cheat to employ here. However, surprisingly,
on some of the easier graphs, the models not only fit the training data, but generalize well to test data.  This positive result (even if in limited settings) verifies two hypotheses. First,
the Clever Hans cheat is indeed caused failure in the original teacher-forced model. Secondly, and remarkably, with the cheat gone, these models are able to fit the first node which had once been indecipherable under teacher-forcing. This  verifies our hypothesis that the Clever Hans cheat absorbs away supervision that is critical to learn the first token. At the end of this section, we provide more intuition for how the absence of Clever Hans cheat, allows the teacherless models to solve this task.

\paragraph{Removing the Indecipherable Token failure via path reversal.} Back in the teacher-forcing setup, we make a slight change: we train the model to predict the reversal of the true path $\response$. Indeed, prior works  \citep{lee23teaching,shen23positional} have proposed reversal in the context of addition tasks as a way of explicitly guiding the next-token predictor to learn a simpler algorithm. Likewise, in our ``reversed'' path-finding task, the model now needs to predict $\goal$ first and make its way to $\start$; the hope is that since there is only one unique path emanating from $\goal$, there is no planning required.  Thus we should never run into an Indecipherable Token. Every next node can be learned as the node that is inwardly adjacent to the previous node. %

We display the results in Fig.~\ref{fig:all-accs} and Table~\ref{table:reverse-gen}. As expected, we observe that reversing  significantly boosts learning, allowing even models trained from scratch to solve the task. This verifies that for the standard model, indecipherability of the first token was indeed a roadblock to successful learning. 

\subsection{Why the Failure of Teacher-Forcing is Remarkable.} 

The success of the reversed training (and  of teacherless training) make the in-distribution failure of teacher-forcing particularly surprising. When viewed left-to-right, our problem requires complex planning --- evaluating multiple paths and selecting the right one --- but when viewed right-to-left, the problem is straightforward; the experiments on the reversed formulation confirm that the right-to-left solution {is} not only expressible by our architectures, but also learnable via gradient descent. Evidently, the left-to-right teacher-forced model is unable to view the problem any differently and falls into the traps outlined in \S\ref{sec:cleverhans}. %

\textbf{Intuition for teacherless training.} We hypothesize that even teacherless training allows the model to implicitly learn the right-to-left view.
Concretely, the teacherless model cannot use the trivial Clever Hans cheat to fit the data, since the ground truth prefixes are not available during training. Nor is it explicitly prescribed to fit the target right-to-left. Instead the model is tasked with using only the graph description in $\prefix$ to fit all the target nodes $\response$ (implicitly requiring a lookahead beyond just the next token). %
Our key intuition is that, in this paradigm, the model would first fit the target token that is simplest to deduce using only information available in the prefix $\prefix$: this is the penultimate vertex $\responsetoken_{l-1}$ which is the unique token that precedes the goal (and can be discovered using a simple scan of the prefix). Once the model figures this out, the model can similarly work backwards to fit each node $\responsetoken_{i-1}$ using the previously-fit $\responsetoken_{i}$.  (Also see Remark~\ref{rem:teacherles}) %

Our hypothesis is borne out in Fig.~\ref{fig:diff-traj} where we see that the later tokens achieve higher accuracy earlier,  implying that the teacherless model voluntarily learns right-to-left. Thus, the teacherless objective provides an alternative training paradigm that forces models to look ahead, without falling into the various short-sighted pitfalls of next-token prediction, discussed in \S\ref{sec:cleverhans}.

\section{Related Work}
\label{sec:related}

We consolidate the arguments surrounding next-token prediction that has been fragmented over various lines of works. 
Part of our elaborate survey is deferred to \S\ref{sec:more-related-work}. 

\textbf{Arguments in support of next-token prediction.} 
\citet{shannon48ntp,shannon51english,alabdulmohsin2024fractal} demonstrate that language has enough redundancy to be conducive for next-token prediction.
Empirically, \citet{shlegeris22language} find that modern language models are surprisingly better than humans at next-token prediction on the text dataset, OpenWebText \citep{Gokaslan2019OpenWeb}. %
But this does not preclude the possibility that next-token predictors may still be poor at planning. %
Furthermore, the above result may be confounded by the ability of language models to store more general knowledge than humans. 

On the theoretical side, \citet{merrill23expressive,feng23towards} show that autoregressive Transformers that generate chains of thought have a larger \textit{expressive} power. %
Most relevant to us is the positive \textit{learnability} results of \citet{malach23auto,weis23subtask} which argue that complex multi-hop tasks that are otherwise unlearnable, become learnable via next-token prediction
when there is a preceding chain-of-thought supervision for each hop. %
Our negative result does not contradict this.
In our problem, learning the first token requires an implicit chain of thought (the reversed path) that we do not provide \textit{before} the first token.

\textbf{Arguments against next-token prediction.} The most well-formulated criticism is the snowballing failure mode, which appears scattered in various forms in literature \citet{dziri23faith,lecun24ar,kaariainen2006lower,ross10efficient}.
As explained earlier, this is orthogonal to our failure;
indeed, in our setting the error happens ``instantaneously'' at the beginning, rather than snowball over time.

Our main counterexample can be seen as formalizing an emerging, informal intuition, often worded as ``autoregressive next-token predictors are ill-suited for planning tasks''.  Indeed, \citet{momennejad23eval,valmeekam23critique,valmeekam2023planbench,planning23valmeekam} report failures on several  planning tasks framed as word problems (including path-finding in \citet{momennejad23eval}) and \citet{bubeck23sparks} on various arithmetic, summarization and poem/story generation tasks.  %
\citet{mccoy23embers} argue that, for such tasks, the performance of the model must greatly depend on its frequency during pretraining. 
However, we show that even when trained on many samples from a distribution, the next-token predictor can fail on the very distribution.

Our work extends and clarifies this discourse by introducing the Clever Hans cheat and the Indecipherable Token failure. Next, we empirically report our failure modes in both the Transformer \citep{vaswani17attention} and the Mamba structured state space model \citep{gu2023mamba}. Thus, what we witness is indeed a failure of next-token prediction (and not of the Transformer architecture as some existing criticisms are framed). Importantly, existing literature pins these failures broadly on the next-token prediction paradigm and interchangeably, on the inability of the autoregressive architecture to backtrack.
We emphasize the need to differentiate between the two types of next-token prediction (teacher-forcing and autoregressive inference) as they lead to distinct planning-related failures and require distinct solutions. %

\textbf{Going beyond next-token prediction.} 
Various works have explored architectures and objectives that train the backbone to go beyond next-token prediction. This includes non-autoregressive models \citep{nag18gu}, energy-based models \cite{dawid23latent}, diffusion models \citep{gong23diffuseq}, and variants of Transformers learning to either predict future tokens in a different ordering \citep{li21discovering} or all at one go \citep{qi20prophet,monea23pass}, or injecting ``lookahead'' data \citep{du23lookahead}.

Teacherless training was proposed for image-to-text captioning models in \citet{tschannen23parallel} under the more generic name of parallel prediction with the  similar motivation that the tokens in a caption must purely rely on the image rather than parts of the caption itself. \citet{monea23pass} proposed the same idea as ``parallel speculative sampling'' for the orthogonal goal improving the inference-time compute. 
On that note, we clarify that research in parallel decoding too is concerned with predicting multiple future tokens \citep{stern18blockwise}, the goal is purely inference-time efficiency. Finally, it is worth noting that action chunking \citep{zhao23act} in imitation learning is a close counterpart of teacherless training although motivated through different arguments.

 \section{Conclusion}

Next-token prediction lies at the heart of modern language models which  have empirically demonstrated tremendous success in wide-ranging tasks. Theoretically too, we know by the chain rule of probability that, 
next-token predictors can express any  distribution over tokens. It is tempting then to view next-token prediction as a formidable approach to modeling intelligence.  Our work crystallizes the core arguments around why this optimism may be misplaced.

We emphasize not to conflate the two modes of next-token prediction: autoregressive inference and teacher-forced training. While existing criticisms primarily challenge autoregressive inference, they assume that teacher-forcing learns a good next-token predictor. We challenge this very assumption, finding that even in a straightforward task, there is failure due to teacher-forcing --- not due to autoregressive inference or the architecture. This casts a shadow over more complex tasks. For instance, as we speculate in \S\ref{sec:story}, can a model trained to predict the next token of thousands of fiction novels, learn to generate plot twists?

An immediate way to circumvent this, as our reversal experiments suggest, is to train with chain-of-thought supervision, echoing \citet{malach23auto,weis23subtask}. However, it is unclear how that is possible in more unstructured tasks like story-writing.
To that end, our minimal counterexample and the idea of {teacherless} training \citep{monea23pass} may inspire alternative paradigms to next-token prediction in practice. Overall, we hope our analyses provide a solid ground to pursue future debates on next-token prediction.

We point the reader to \S\ref{sec:limitations} for a discussion of the limitations of our study.

\section*{Acknowledgments} 

We would like to thank Colin Raffel, Tiago Pimentel, and Surbhi Goel for their extensive feedback on a draft of this preprint, especially for pointers to some key references.
We thank Garret Tanzer for pointing us to \citet{tschannen23parallel} for a prior instance of the idea of teacherless training.

\section*{Impact Statement}

Our results outline the limits of a foundational technique that lies at the heart of modern AI systems. Naturally, there are many potential downstream societal consequences that would apply at large to such foundational work, none we feel must be specifically highlighted here.

\bibliography{references}

\begin{thebibliography}{106}
\providecommand{\natexlab}[1]{#1}
\providecommand{\url}[1]{\texttt{#1}}
\expandafter\ifx\csname urlstyle\endcsname\relax
  \providecommand{\doi}[1]{doi: #1}\else
  \providecommand{\doi}{doi: \begingroup \urlstyle{rm}\Url}\fi

\bibitem[Alabdulmohsin et~al.(2024)Alabdulmohsin, Tran, and Dehghani]{alabdulmohsin2024fractal}
Alabdulmohsin, I., Tran, V.~Q., and Dehghani, M.
\newblock Fractal patterns may unravel the intelligence in next-token prediction, 2024.

\bibitem[Allen-Zhu \& Li(2023)Allen-Zhu and Li]{zhu23physics}
Allen-Zhu, Z. and Li, Y.
\newblock Physics of language models: Part 3.2, knowledge manipulation.
\newblock \emph{arXiv preprint arXiv:2309.14402}, 2023.

\bibitem[Anil et~al.(2022)Anil, Wu, Andreassen, Lewkowycz, Misra, Ramasesh, Slone, Gur{-}Ari, Dyer, and Neyshabur]{anil22length}
Anil, C., Wu, Y., Andreassen, A., Lewkowycz, A., Misra, V., Ramasesh, V.~V., Slone, A., Gur{-}Ari, G., Dyer, E., and Neyshabur, B.
\newblock Exploring length generalization in large language models.
\newblock In \emph{Advances in Neural Information Processing Systems 35: Annual Conference on Neural Information Processing Systems 2022, NeurIPS 2022}, 2022.

\bibitem[Arkoudas(2023)]{arkoudas2023chatgpt}
Arkoudas, K.
\newblock Chatgpt is no stochastic parrot. but it also claims that 1 is greater than 1.
\newblock \emph{Philosophy \& Technology}, 36\penalty0 (3):\penalty0 54, 2023.

\bibitem[Artetxe et~al.(2022)Artetxe, Du, Goyal, Zettlemoyer, and Stoyanov]{artetxe22role}
Artetxe, M., Du, J., Goyal, N., Zettlemoyer, L., and Stoyanov, V.
\newblock On the role of bidirectionality in language model pre-training.
\newblock In \emph{Findings of the Association for Computational Linguistics: {EMNLP} 2022, Abu Dhabi, United Arab Emirates, December 7-11, 2022}, pp.\  3973--3985. Association for Computational Linguistics, 2022.

\bibitem[Bahdanau et~al.(2017)Bahdanau, Brakel, Xu, Goyal, Lowe, Pineau, Courville, and Bengio]{bahdanau17actor}
Bahdanau, D., Brakel, P., Xu, K., Goyal, A., Lowe, R., Pineau, J., Courville, A.~C., and Bengio, Y.
\newblock An actor-critic algorithm for sequence prediction.
\newblock In \emph{5th International Conference on Learning Representations, {ICLR} 2017, Conference Track Proceedings}, 2017.

\bibitem[Bender et~al.(2021)Bender, Gebru, McMillan{-}Major, and Shmitchell]{bender21stochastic}
Bender, E.~M., Gebru, T., McMillan{-}Major, A., and Shmitchell, S.
\newblock On the dangers of stochastic parrots: Can language models be too big?
\newblock In \emph{FAccT '21: 2021 {ACM} Conference on Fairness, Accountability, and Transparency, Virtual Event / Toronto, Canada, March 3-10, 2021}, pp.\  610--623. {ACM}, 2021.

\bibitem[Bengio et~al.(2015)Bengio, Vinyals, Jaitly, and Shazeer]{bengio15scheduled}
Bengio, S., Vinyals, O., Jaitly, N., and Shazeer, N.
\newblock Scheduled sampling for sequence prediction with recurrent neural networks.
\newblock In \emph{Advances in Neural Information Processing Systems 28: Annual Conference on Neural Information Processing Systems 2015}, pp.\  1171--1179, 2015.

\bibitem[Besta et~al.(2024)Besta, Blach, Kubicek, Gerstenberger, Podstawski, Gianinazzi, Gajda, Lehmann, Niewiadomski, Nyczyk, et~al.]{besta23graph}
Besta, M., Blach, N., Kubicek, A., Gerstenberger, R., Podstawski, M., Gianinazzi, L., Gajda, J., Lehmann, T., Niewiadomski, H., Nyczyk, P., et~al.
\newblock Graph of thoughts: Solving elaborate problems with large language models.
\newblock In \emph{Proceedings of the AAAI Conference on Artificial Intelligence}, volume~38, 2024.

\bibitem[Bubeck et~al.(2023)Bubeck, Chandrasekaran, Eldan, Gehrke, Horvitz, Kamar, Lee, Lee, Li, Lundberg, et~al.]{bubeck23sparks}
Bubeck, S., Chandrasekaran, V., Eldan, R., Gehrke, J., Horvitz, E., Kamar, E., Lee, P., Lee, Y.~T., Li, Y., Lundberg, S., et~al.
\newblock Sparks of artificial general intelligence: Early experiments with gpt-4.
\newblock \emph{arXiv preprint arXiv:2303.12712}, 2023.

\bibitem[Burtsev et~al.(2020)Burtsev, Kuratov, Peganov, and Sapunov]{burtsev2020memory}
Burtsev, M.~S., Kuratov, Y., Peganov, A., and Sapunov, G.~V.
\newblock Memory transformer.
\newblock \emph{arXiv preprint arXiv:2006.11527}, 2020.

\bibitem[Chang et~al.(2015)Chang, Krishnamurthy, Agarwal, III, and Langford]{chang15lols}
Chang, K., Krishnamurthy, A., Agarwal, A., III, H.~D., and Langford, J.
\newblock Learning to search better than your teacher.
\newblock In \emph{Proceedings of the 32nd International Conference on Machine Learning, {ICML} 2015}, volume~37 of \emph{{JMLR} Workshop and Conference Proceedings}, 2015.

\bibitem[Cobbe et~al.(2021)Cobbe, Kosaraju, Bavarian, Chen, Jun, Kaiser, Plappert, Tworek, Hilton, Nakano, Hesse, and Schulman]{cobbe21training}
Cobbe, K., Kosaraju, V., Bavarian, M., Chen, M., Jun, H., Kaiser, L., Plappert, M., Tworek, J., Hilton, J., Nakano, R., Hesse, C., and Schulman, J.
\newblock Training verifiers to solve math word problems.
\newblock \emph{arXiv preprint arXiv:2110.14168}, 2021.

\bibitem[{Daum{\'{e}} III} et~al.(2009){Daum{\'{e}} III}, Langford, and Marcu]{daume09structured}
{Daum{\'{e}} III}, H., Langford, J., and Marcu, D.
\newblock Search-based structured prediction.
\newblock \emph{Mach. Learn.}, 75\penalty0 (3):\penalty0 297--325, 2009.

\bibitem[Dawid \& LeCun(2023)Dawid and LeCun]{dawid23latent}
Dawid, A. and LeCun, Y.
\newblock Introduction to latent variable energy-based models: A path towards autonomous machine intelligence.
\newblock \emph{arXiv preprint arXiv:2306.02572}, 2023.

\bibitem[Du et~al.(2023{\natexlab{a}})Du, Mei, and Eisner]{du23lookahead}
Du, L., Mei, H., and Eisner, J.
\newblock Autoregressive modeling with lookahead attention.
\newblock \emph{arXiv preprint arXiv:2305.12272}, 2023{\natexlab{a}}.

\bibitem[Du et~al.(2023{\natexlab{b}})Du, Torroba~Hennigen, Pimentel, Meister, Eisner, and Cotterell]{du2023measure}
Du, L., Torroba~Hennigen, L., Pimentel, T., Meister, C., Eisner, J., and Cotterell, R.
\newblock A measure-theoretic characterization of tight language models.
\newblock In \emph{Proceedings of the 61st Annual Meeting of the Association for Computational Linguistics (Volume 1: Long Papers)}, Toronto, Canada, 2023{\natexlab{b}}. Association for Computational Linguistics.

\bibitem[Dziri et~al.(2024)Dziri, Lu, Sclar, Li, Jiang, Lin, Welleck, West, Bhagavatula, Le~Bras, et~al.]{dziri23faith}
Dziri, N., Lu, X., Sclar, M., Li, X.~L., Jiang, L., Lin, B.~Y., Welleck, S., West, P., Bhagavatula, C., Le~Bras, R., et~al.
\newblock Faith and fate: Limits of transformers on compositionality.
\newblock \emph{Advances in Neural Information Processing Systems}, 36, 2024.

\bibitem[Feng et~al.(2023)Feng, Zhang, Gu, Ye, He, and Wang]{feng23towards}
Feng, G., Zhang, B., Gu, Y., Ye, H., He, D., and Wang, L.
\newblock Towards revealing the mystery behind chain of thought: a theoretical perspective.
\newblock \emph{Advances in Neural Information Processing Systems}, 36, 2023.

\bibitem[Glasmachers(2017)]{glasmachers17endtoend}
Glasmachers, T.
\newblock Limits of end-to-end learning.
\newblock In Zhang, M. and Noh, Y. (eds.), \emph{Proceedings of The 9th Asian Conference on Machine Learning, {ACML} 2017}, volume~77 of \emph{Proceedings of Machine Learning Research}, pp.\  17--32. {PMLR}, 2017.

\bibitem[Gokaslan \& Cohen(2019)Gokaslan and Cohen]{Gokaslan2019OpenWeb}
Gokaslan, A. and Cohen, V.
\newblock Openwebtext corpus.
\newblock \url{http://Skylion007.github.io/OpenWebTextCorpus}, 2019.

\bibitem[Gong et~al.(2023)Gong, Li, Feng, Wu, and Kong]{gong23diffuseq}
Gong, S., Li, M., Feng, J., Wu, Z., and Kong, L.
\newblock Diffuseq: Sequence to sequence text generation with diffusion models.
\newblock In \emph{The Eleventh International Conference on Learning Representations, {ICLR} 2023, Kigali, Rwanda, May 1-5, 2023}. OpenReview.net, 2023.

\bibitem[Goyal et~al.(2016)Goyal, Lamb, Zhang, Zhang, Courville, and Bengio]{goyal16professor}
Goyal, A., Lamb, A., Zhang, Y., Zhang, S., Courville, A.~C., and Bengio, Y.
\newblock Professor forcing: {A} new algorithm for training recurrent networks.
\newblock In \emph{Advances in Neural Information Processing Systems 29: Annual Conference on Neural Information Processing Systems 2016}, pp.\  4601--4609, 2016.

\bibitem[Goyal et~al.(20234)Goyal, Ji, Rawat, Menon, Kumar, and Nagarajan]{goyal23think}
Goyal, S., Ji, Z., Rawat, A.~S., Menon, A.~K., Kumar, S., and Nagarajan, V.
\newblock Think before you speak: Training language models with pause tokens.
\newblock \emph{The Twelfth International Conference on Learning Representations, {ICLR} 2024}, 20234.

\bibitem[Gu \& Dao(2023)Gu and Dao]{gu2023mamba}
Gu, A. and Dao, T.
\newblock Mamba: Linear-time sequence modeling with selective state spaces, 2023.

\bibitem[Gu et~al.(2018)Gu, Bradbury, Xiong, Li, and Socher]{nag18gu}
Gu, J., Bradbury, J., Xiong, C., Li, V. O.~K., and Socher, R.
\newblock Non-autoregressive neural machine translation.
\newblock In \emph{6th International Conference on Learning Representations, {ICLR} 2018, Conference Track Proceedings}. OpenReview.net, 2018.

\bibitem[G{\"{u}}l{\c{c}}ehre \& Bengio(2016)G{\"{u}}l{\c{c}}ehre and Bengio]{gulcehre16knowledge}
G{\"{u}}l{\c{c}}ehre, {\c{C}}. and Bengio, Y.
\newblock Knowledge matters: Importance of prior information for optimization.
\newblock \emph{J. Mach. Learn. Res.}, 17:\penalty0 8:1--8:32, 2016.

\bibitem[Gurnee et~al.(2023)Gurnee, Nanda, Pauly, Harvey, Troitskii, and Bertsimas]{gurnee10finding}
Gurnee, W., Nanda, N., Pauly, M., Harvey, K., Troitskii, D., and Bertsimas, D.
\newblock Finding neurons in a haystack: Case studies with sparse probing.
\newblock \emph{arXiv preprint arXiv:2305.01610}, 2023.

\bibitem[Havrilla et~al.(2024)Havrilla, Du, Raparthy, Nalmpantis, Dwivedi-Yu, Zhuravinskyi, Hambro, Sukhbaatar, and Raileanu]{havrilla2024teaching}
Havrilla, A., Du, Y., Raparthy, S.~C., Nalmpantis, C., Dwivedi-Yu, J., Zhuravinskyi, M., Hambro, E., Sukhbaatar, S., and Raileanu, R.
\newblock Teaching large language models to reason with reinforcement learning, 2024.

\bibitem[Hsieh et~al.(2023)Hsieh, Li, Yeh, Nakhost, Fujii, Ratner, Krishna, Lee, and Pfister]{hsieh2023distilling}
Hsieh, C.-Y., Li, C.-L., Yeh, C.-K., Nakhost, H., Fujii, Y., Ratner, A., Krishna, R., Lee, C.-Y., and Pfister, T.
\newblock Distilling step-by-step! outperforming larger language models with less training data and smaller model sizes.
\newblock \emph{arXiv preprint arXiv:2305.02301}, 2023.

\bibitem[Huang et~al.(2022)Huang, Xia, Xiao, Chan, Liang, Florence, Zeng, Tompson, Mordatch, Chebotar, Sermanet, Jackson, Brown, Luu, Levine, Hausman, and Ichter]{huang22internal}
Huang, W., Xia, F., Xiao, T., Chan, H., Liang, J., Florence, P., Zeng, A., Tompson, J., Mordatch, I., Chebotar, Y., Sermanet, P., Jackson, T., Brown, N., Luu, L., Levine, S., Hausman, K., and Ichter, B.
\newblock Inner monologue: Embodied reasoning through planning with language models.
\newblock In \emph{Conference on Robot Learning, CoRL 2022, 14-18 December 2022, Auckland, New Zealand}, volume 205 of \emph{Proceedings of Machine Learning Research}, pp.\  1769--1782. {PMLR}, 2022.

\bibitem[Jiang et~al.(2023)Jiang, Sablayrolles, Mensch, Bamford, Chaplot, de~las Casas, Bressand, Lengyel, Lample, Saulnier, Lavaud, Lachaux, Stock, Scao, Lavril, Wang, Lacroix, and Sayed]{jiang2023mistral}
Jiang, A.~Q., Sablayrolles, A., Mensch, A., Bamford, C., Chaplot, D.~S., de~las Casas, D., Bressand, F., Lengyel, G., Lample, G., Saulnier, L., Lavaud, L.~R., Lachaux, M.-A., Stock, P., Scao, T.~L., Lavril, T., Wang, T., Lacroix, T., and Sayed, W.~E.
\newblock Mistral 7b, 2023.

\bibitem[K{\"a}{\"a}ri{\"a}inen(2006)]{kaariainen2006lower}
K{\"a}{\"a}ri{\"a}inen, M.
\newblock Lower bounds for reductions.
\newblock In \emph{Atomic Learning Workshop}, 2006.

\bibitem[Kahneman(2011)]{kahneman2011thinking}
Kahneman, D.
\newblock \emph{Thinking, fast and slow}.
\newblock Farrar, Straus and Giroux, 2011.

\bibitem[Kojima et~al.(2022)Kojima, Gu, Reid, Matsuo, and Iwasawa]{kojima2022large}
Kojima, T., Gu, S.~S., Reid, M., Matsuo, Y., and Iwasawa, Y.
\newblock Large language models are zero-shot reasoners.
\newblock \emph{Advances in neural information processing systems}, 35, 2022.

\bibitem[Lai et~al.(2021)Lai, Zhang, Feng, Huang, and Zhao]{mrc21lai}
Lai, Y., Zhang, C., Feng, Y., Huang, Q., and Zhao, D.
\newblock Why machine reading comprehension models learn shortcuts?
\newblock In \emph{Findings of the Association for Computational Linguistics: {ACL/IJCNLP} 2021, Online Event, August 1-6, 2021}, volume {ACL/IJCNLP} 2021 of \emph{Findings of {ACL}}, pp.\  989--1002. Association for Computational Linguistics, 2021.

\bibitem[LeCun(2024)]{lecun24ar}
LeCun, Y.
\newblock Do large language models need sensory grounding for meaning and understanding?
\newblock University Lecture, 2024.

\bibitem[Lee et~al.(2023)Lee, Sreenivasan, Lee, Lee, and Papailiopoulos]{lee23teaching}
Lee, N., Sreenivasan, K., Lee, J.~D., Lee, K., and Papailiopoulos, D.
\newblock Teaching arithmetic to small transformers.
\newblock \emph{arXiv preprint arXiv:2307.03381}, 2023.

\bibitem[Li et~al.(2021)Li, Trabucco, Park, Luo, Shen, Darrell, and Gao]{li21discovering}
Li, X., Trabucco, B., Park, D.~H., Luo, M., Shen, S., Darrell, T., and Gao, Y.
\newblock Discovering non-monotonic autoregressive orderings with variational inference.
\newblock In \emph{9th International Conference on Learning Representations, {ICLR} 2021, Virtual Event, Austria, May 3-7, 2021}. OpenReview.net, 2021.
\newblock URL \url{https://openreview.net/forum?id=jP1vTH3inC}.

\bibitem[Li et~al.(2024)Li, Huang, Ildiz, Rawat, and Oymak]{li24ntp}
Li, Y., Huang, Y., Ildiz, M.~E., Rawat, A.~S., and Oymak, S.
\newblock Mechanics of next token prediction with self-attention.
\newblock In \emph{27th International Conference on Artificial Intelligence and Statistics (AISTATS)}, 2024.

\bibitem[Lin et~al.(2021)Lin, Jaech, Li, Gormley, and Eisner]{lin21limitations}
Lin, C., Jaech, A., Li, X., Gormley, M.~R., and Eisner, J.
\newblock Limitations of autoregressive models and their alternatives.
\newblock In \emph{Proceedings of the 2021 Conference of the North American Chapter of the Association for Computational Linguistics: Human Language Technologies, {NAACL-HLT} 2021, Online, June 6-11, 2021}, pp.\  5147--5173. Association for Computational Linguistics, 2021.

\bibitem[Ling et~al.(2017)Ling, Yogatama, Dyer, and Blunsom]{ling17induction}
Ling, W., Yogatama, D., Dyer, C., and Blunsom, P.
\newblock Program induction by rationale generation: Learning to solve and explain algebraic word problems.
\newblock In \emph{Proceedings of the 55th Annual Meeting of the Association for Computational Linguistics, {ACL} 2017, Vancouver, Canada, July 30 - August 4, Volume 1: Long Papers}, pp.\  158--167. Association for Computational Linguistics, 2017.

\bibitem[Liu et~al.(2023)Liu, Ash, Goel, Krishnamurthy, and Zhang]{liu23automata}
Liu, B., Ash, J.~T., Goel, S., Krishnamurthy, A., and Zhang, C.
\newblock Transformers learn shortcuts to automata.
\newblock In \emph{The Eleventh International Conference on Learning Representations, {ICLR} 2023}, 2023.

\bibitem[Liu et~al.(2020)]{stiennon20rl}
Liu, F. et~al.
\newblock Learning to summarize from human feedback.
\newblock In \emph{Proceedings of the 58th Annual Meeting of the Association for Computational Linguistics}, 2020.

\bibitem[Loshchilov \& Hutter(2019)Loshchilov and Hutter]{loshchilov2018decoupled}
Loshchilov, I. and Hutter, F.
\newblock Decoupled weight decay regularization.
\newblock In \emph{7th International Conference on Learning Representations, {ICLR} 2019, New Orleans, LA, USA, May 6-9, 2019}. OpenReview.net, 2019.

\bibitem[Lv et~al.(2023)Lv, Zhang, Xie, Tu, Chen, Wen, and Yan]{lv23reversal}
Lv, A., Zhang, K., Xie, S., Tu, Q., Chen, Y., Wen, J.-R., and Yan, R.
\newblock Are we falling in a middle-intelligence trap? an analysis and mitigation of the reversal curse.
\newblock \emph{arXiv preprint arXiv:2311.07468}, 2023.

\bibitem[Madaan et~al.(2024)Madaan, Tandon, Gupta, Hallinan, Gao, Wiegreffe, Alon, Dziri, Prabhumoye, Yang, et~al.]{madan23selfrefine}
Madaan, A., Tandon, N., Gupta, P., Hallinan, S., Gao, L., Wiegreffe, S., Alon, U., Dziri, N., Prabhumoye, S., Yang, Y., et~al.
\newblock Self-refine: Iterative refinement with self-feedback.
\newblock \emph{Advances in Neural Information Processing Systems}, 36, 2024.

\bibitem[Malach(2023)]{malach23auto}
Malach, E.
\newblock Auto-regressive next-token predictors are universal learners.
\newblock \emph{arXiv preprint arXiv:2309.06979}, 2023.

\bibitem[McCoy et~al.(2023)McCoy, Yao, Friedman, Hardy, and Griffiths]{mccoy23embers}
McCoy, R.~T., Yao, S., Friedman, D., Hardy, M., and Griffiths, T.~L.
\newblock Embers of autoregression: Understanding large language models through the problem they are trained to solve.
\newblock \emph{arXiv preprint arXiv:2309.13638}, 2023.

\bibitem[Meng et~al.(2022)Meng, Bau, Andonian, and Belinkov]{meng22locating}
Meng, K., Bau, D., Andonian, A., and Belinkov, Y.
\newblock Locating and editing factual associations in {GPT}.
\newblock In \emph{Advances in Neural Information Processing Systems 35: Annual Conference on Neural Information Processing Systems 2022, NeurIPS 2022}, 2022.

\bibitem[Merrill \& Sabharwal(2023)Merrill and Sabharwal]{merrill2023parallelism}
Merrill, W. and Sabharwal, A.
\newblock The parallelism tradeoff: Limitations of log-precision transformers, 2023.

\bibitem[Merrill \& Sabharwal(2024)Merrill and Sabharwal]{merrill23expressive}
Merrill, W. and Sabharwal, A.
\newblock The expressive power of transformers with chain of thought.
\newblock In \emph{The Twelfth International Conference on Learning Representations}, 2024.

\bibitem[Momennejad et~al.(2023)Momennejad, Hasanbeig, Frujeri, Sharma, Ness, Jojic, Palangi, and Larson]{momennejad23eval}
Momennejad, I., Hasanbeig, H., Frujeri, F.~V., Sharma, H., Ness, R.~O., Jojic, N., Palangi, H., and Larson, J.
\newblock Evaluating cognitive maps and planning in large language models with cogeval.
\newblock \emph{Advances in Neural Information Processing Systems}, 36, 2023.

\bibitem[Monea et~al.(2023)Monea, Joulin, and Grave]{monea23pass}
Monea, G., Joulin, A., and Grave, E.
\newblock Pass: Parallel speculative sampling.
\newblock \emph{3rd Workshop on Efficient Natural Language and Speech Processing (NeurIPS 2023)}, 2023.

\bibitem[Nye et~al.(2021)Nye, Andreassen, Gur{-}Ari, Michalewski, Austin, Bieber, Dohan, Lewkowycz, Bosma, Luan, Sutton, and Odena]{nye21scratchpad}
Nye, M.~I., Andreassen, A.~J., Gur{-}Ari, G., Michalewski, H., Austin, J., Bieber, D., Dohan, D., Lewkowycz, A., Bosma, M., Luan, D., Sutton, C., and Odena, A.
\newblock Show your work: Scratchpads for intermediate computation with language models.
\newblock \emph{arXiv preprint arXiv:2112.00114}, 2021.

\bibitem[Olsson et~al.(2022)Olsson, Elhage, Nanda, Joseph, DasSarma, Henighan, Mann, Askell, Bai, Chen, et~al.]{olsson22icl}
Olsson, C., Elhage, N., Nanda, N., Joseph, N., DasSarma, N., Henighan, T., Mann, B., Askell, A., Bai, Y., Chen, A., et~al.
\newblock In-context learning and induction heads.
\newblock \emph{arXiv preprint arXiv:2209.11895}, 2022.

\bibitem[Ouyang et~al.(2022)Ouyang, Wu, Jiang, Almeida, Wainwright, Mishkin, Zhang, Agarwal, Slama, Ray, Schulman, Hilton, Kelton, Miller, Simens, Askell, Welinder, Christiano, Leike, and Lowe]{ouyang19rl}
Ouyang, L., Wu, J., Jiang, X., Almeida, D., Wainwright, C.~L., Mishkin, P., Zhang, C., Agarwal, S., Slama, K., Ray, A., Schulman, J., Hilton, J., Kelton, F., Miller, L., Simens, M., Askell, A., Welinder, P., Christiano, P.~F., Leike, J., and Lowe, R.
\newblock Training language models to follow instructions with human feedback.
\newblock In \emph{Advances in Neural Information Processing Systems 35: Annual Conference on Neural Information Processing Systems 2022, NeurIPS 2022}, 2022.

\bibitem[Pal et~al.(2023)Pal, Sun, Yuan, Wallace, and Bau]{pal23future}
Pal, K., Sun, J., Yuan, A., Wallace, B.~C., and Bau, D.
\newblock Future lens: Anticipating subsequent tokens from a single hidden state.
\newblock In \emph{Proceedings of the 27th Conference on Computational Natural Language Learning, CoNLL 2023}. Association for Computational Linguistics, 2023.

\bibitem[Papadopoulos et~al.(2024)Papadopoulos, Wenger, and Hongler]{papadopoulos2024arrows}
Papadopoulos, V., Wenger, J., and Hongler, C.
\newblock Arrows of time for large language models, 2024.

\bibitem[Paulus et~al.(2018)Paulus, Xiong, and Socher]{paulus18deep}
Paulus, R., Xiong, C., and Socher, R.
\newblock A deep reinforced model for abstractive summarization.
\newblock In \emph{6th International Conference on Learning Representations, {ICLR} 2018, Conference Track Proceedings}. OpenReview.net, 2018.

\bibitem[Pezeshki et~al.(2021)Pezeshki, Kaba, Bengio, Courville, Precup, and Lajoie]{pezeshki21starvation}
Pezeshki, M., Kaba, S., Bengio, Y., Courville, A.~C., Precup, D., and Lajoie, G.
\newblock Gradient starvation: {A} learning proclivity in neural networks.
\newblock In \emph{Advances in Neural Information Processing Systems 34: Annual Conference on Neural Information Processing Systems 2021, NeurIPS 2021, December 6-14, 2021, virtual}, pp.\  1256--1272, 2021.

\bibitem[Pfau et~al.(2023)Pfau, Infanger, Sheshadri, Panda, Michael, and Huebner]{pfau2023eliciting}
Pfau, J., Infanger, A., Sheshadri, A., Panda, A., Michael, J., and Huebner, C.
\newblock Eliciting language model behaviors using reverse language models.
\newblock In \emph{Socially Responsible Language Modelling Research}, 2023.

\bibitem[Pfungst \& Rahn(1911)Pfungst and Rahn]{bhlitem116908}
Pfungst, O. and Rahn, C.~L.
\newblock \emph{Clever Hans (the horse of Mr. Von Osten) a contribution to experimental animal and human psychology}.
\newblock New York, H. Holt and company, 1911.

\bibitem[Piekos et~al.(2021)Piekos, Malinowski, and Michalewski]{piekos20measuring}
Piekos, P., Malinowski, M., and Michalewski, H.
\newblock Measuring and improving bert's mathematical abilities by predicting the order of reasoning.
\newblock In \emph{Proceedings of the 59th Annual Meeting of the Association for Computational Linguistics and the 11th International Joint Conference on Natural Language Processing, {ACL/IJCNLP} 2021, (Volume 2: Short Papers), Virtual Event, August 1-6, 2021}, pp.\  383--394. Association for Computational Linguistics, 2021.

\bibitem[Power et~al.(2022)Power, Burda, Edwards, Babuschkin, and Misra]{power2022grokking}
Power, A., Burda, Y., Edwards, H., Babuschkin, I., and Misra, V.
\newblock Grokking: Generalization beyond overfitting on small algorithmic datasets, 2022.

\bibitem[Qi et~al.(2020)Qi, Yan, Gong, Liu, Duan, Chen, Zhang, and Zhou]{qi20prophet}
Qi, W., Yan, Y., Gong, Y., Liu, D., Duan, N., Chen, J., Zhang, R., and Zhou, M.
\newblock Prophetnet: Predicting future n-gram for sequence-to-sequence pre-training.
\newblock In \emph{Findings of the Association for Computational Linguistics: {EMNLP} 2020, Online Event, 16-20 November 2020}, volume {EMNLP} 2020 of \emph{Findings of {ACL}}, pp.\  2401--2410, 2020.

\bibitem[Radford et~al.(2019)Radford, Wu, Child, Luan, Amodei, and Sutskever]{Radford2019LanguageMA}
Radford, A., Wu, J., Child, R., Luan, D., Amodei, D., and Sutskever, I.
\newblock Language models are unsupervised multitask learners.
\newblock 2019.

\bibitem[Ranaldi \& Zanzotto(2023)Ranaldi and Zanzotto]{ranaldi2023hans}
Ranaldi, L. and Zanzotto, F.~M.
\newblock Hans, are you clever? clever hans effect analysis of neural systems, 2023.

\bibitem[Ranzato et~al.(2016)Ranzato, Chopra, Auli, and Zaremba]{ranzato16seq}
Ranzato, M., Chopra, S., Auli, M., and Zaremba, W.
\newblock Sequence level training with recurrent neural networks.
\newblock In \emph{4th International Conference on Learning Representations, {ICLR} 2016, Conference Track Proceedings}, 2016.

\bibitem[Recchia(2021)]{recchia21teaching}
Recchia, G.
\newblock Teaching autoregressive language models complex tasks by demonstration.
\newblock \emph{arXiv preprint arXiv:2109.02102}, 2021.

\bibitem[Reynolds \& McDonell(2021)Reynolds and McDonell]{reynolds2021prompt}
Reynolds, L. and McDonell, K.
\newblock Prompt programming for large language models: Beyond the few-shot paradigm.
\newblock In \emph{Extended Abstracts of the 2021 CHI Conference on Human Factors in Computing Systems}, 2021.

\bibitem[Ross \& Bagnell(2010)Ross and Bagnell]{ross10efficient}
Ross, S. and Bagnell, D.
\newblock Efficient reductions for imitation learning.
\newblock In Teh, Y.~W. and Titterington, D.~M. (eds.), \emph{Proceedings of the Thirteenth International Conference on Artificial Intelligence and Statistics, {AISTATS} 2010, Chia Laguna Resort, Sardinia, Italy, May 13-15, 2010}, volume~9 of \emph{{JMLR} Proceedings}, 2010.

\bibitem[Ross \& Bagnell(2014)Ross and Bagnell]{ross14aggrevate}
Ross, S. and Bagnell, J.~A.
\newblock Reinforcement and imitation learning via interactive no-regret learning.
\newblock abs/1406.5979, 2014.

\bibitem[Ross et~al.(2011)Ross, Gordon, and Bagnell]{ross11reduction}
Ross, S., Gordon, G.~J., and Bagnell, D.
\newblock A reduction of imitation learning and structured prediction to no-regret online learning.
\newblock In \emph{Proceedings of the Fourteenth International Conference on Artificial Intelligence and Statistics, {AISTATS} 2011}, {JMLR} Proceedings, 2011.

\bibitem[Sanford et~al.(2024)Sanford, Hsu, and Telgarsky]{sanford2024transformers}
Sanford, C., Hsu, D., and Telgarsky, M.
\newblock Transformers, parallel computation, and logarithmic depth.
\newblock \emph{arXiv preprint arXiv:2402.09268}, 2024.

\bibitem[Shah et~al.(2020)Shah, Tamuly, Raghunathan, Jain, and Netrapalli]{shah20simplicity}
Shah, H., Tamuly, K., Raghunathan, A., Jain, P., and Netrapalli, P.
\newblock The pitfalls of simplicity bias in neural networks.
\newblock In \emph{Advances in Neural Information Processing Systems 33: Annual Conference on Neural Information Processing Systems 2020, NeurIPS 2020, December 6-12, 2020, virtual}, 2020.

\bibitem[Shalev{-}Shwartz \& Shashua(2016)Shalev{-}Shwartz and Shashua]{shwartz16endtoend}
Shalev{-}Shwartz, S. and Shashua, A.
\newblock On the sample complexity of end-to-end training vs. semantic abstraction training.
\newblock \emph{arXiv preprint arXiv:1604.06915}, 2016.

\bibitem[Shalev{-}Shwartz et~al.(2017)Shalev{-}Shwartz, Shamir, and Shammah]{shwartz17failures}
Shalev{-}Shwartz, S., Shamir, O., and Shammah, S.
\newblock Failures of gradient-based deep learning.
\newblock In \emph{Proceedings of the 34th International Conference on Machine Learning, {ICML} 2017, Sydney, NSW, Australia, 6-11 August 2017}, volume~70 of \emph{Proceedings of Machine Learning Research}, pp.\  3067--3075. {PMLR}, 2017.

\bibitem[Shannon(1948)]{shannon48ntp}
Shannon, C.~E.
\newblock A mathematical theory of communication.
\newblock \emph{The Bell System Technical Journal}, 27\penalty0 (3):\penalty0 379--423, 1948.

\bibitem[Shannon(1951)]{shannon51english}
Shannon, C.~E.
\newblock Prediction and entropy of printed english.
\newblock \emph{The Bell System Technical Journal}, 30\penalty0 (1):\penalty0 50--64, 1951.

\bibitem[Shen et~al.(2023)Shen, Bubeck, Eldan, Lee, Li, and Zhang]{shen23positional}
Shen, R., Bubeck, S., Eldan, R., Lee, Y.~T., Li, Y., and Zhang, Y.
\newblock Positional description matters for transformers arithmetic.
\newblock \emph{arXiv preprint arXiv:2311.14737}, 2023.

\bibitem[Shinn et~al.(2023)Shinn, Cassano, Berman, Gopinath, Narasimhan, and Yao]{shinn2023reflexion}
Shinn, N., Cassano, F., Berman, E., Gopinath, A., Narasimhan, K., and Yao, S.
\newblock Reflexion: Language agents with verbal reinforcement learning, 2023.

\bibitem[Shlegeris et~al.(2022)Shlegeris, Roger, Chan, and McLean]{shlegeris22language}
Shlegeris, B., Roger, F., Chan, L., and McLean, E.
\newblock Language models are better than humans at next-token prediction.
\newblock \emph{arXiv preprint arXiv:2212.11281}, 2022.

\bibitem[Shridhar et~al.(2022)Shridhar, Stolfo, and Sachan]{shridhar2022distilling}
Shridhar, K., Stolfo, A., and Sachan, M.
\newblock Distilling reasoning capabilities into smaller language models.
\newblock \emph{arXiv preprint arXiv:2212.00193}, 2022.

\bibitem[Springer et~al.(2024)Springer, Kotha, Fried, Neubig, and Raghunathan]{springer2024repetition}
Springer, J.~M., Kotha, S., Fried, D., Neubig, G., and Raghunathan, A.
\newblock Repetition improves language model embeddings, 2024.

\bibitem[Stern et~al.(2018)Stern, Shazeer, and Uszkoreit]{stern18blockwise}
Stern, M., Shazeer, N., and Uszkoreit, J.
\newblock Blockwise parallel decoding for deep autoregressive models.
\newblock In \emph{Advances in Neural Information Processing Systems 31: Annual Conference on Neural Information Processing Systems 2018, NeurIPS 2018, December 3-8, 2018, Montr{\'{e}}al, Canada}, 2018.

\bibitem[Thrampoulidis(2024)]{thrampoulidis2024implicit}
Thrampoulidis, C.
\newblock Implicit bias of next-token prediction, 2024.

\bibitem[Touvron et~al.(2023)Touvron, Martin, Stone, Albert, Almahairi, Babaei, Bashlykov, Batra, Bhargava, Bhosale, et~al.]{touvron2023llama}
Touvron, H., Martin, L., Stone, K., Albert, P., Almahairi, A., Babaei, Y., Bashlykov, N., Batra, S., Bhargava, P., Bhosale, S., et~al.
\newblock Llama 2: Open foundation and fine-tuned chat models.
\newblock \emph{arXiv preprint arXiv:2307.09288}, 2023.

\bibitem[Tschannen et~al.(2023)Tschannen, Kumar, Steiner, Zhai, Houlsby, and Beyer]{tschannen23parallel}
Tschannen, M., Kumar, M., Steiner, A., Zhai, X., Houlsby, N., and Beyer, L.
\newblock Image captioners are scalable vision learners too.
\newblock In \emph{Advances in Neural Information Processing Systems 36: Annual Conference on Neural Information Processing Systems 2023, NeurIPS 2023, New Orleans, LA, USA, December 10 - 16, 2023}, 2023.

\bibitem[Valmeekam et~al.(2023{\natexlab{a}})Valmeekam, Marquez, and Kambhampati]{valmeekam23critique}
Valmeekam, K., Marquez, M., and Kambhampati, S.
\newblock Can large language models really improve by self-critiquing their own plans?
\newblock \emph{arXiv preprint arXiv:2310.08118}, 2023{\natexlab{a}}.

\bibitem[Valmeekam et~al.(2023{\natexlab{b}})Valmeekam, Marquez, Olmo, Sreedharan, and Kambhampati]{valmeekam2023planbench}
Valmeekam, K., Marquez, M., Olmo, A., Sreedharan, S., and Kambhampati, S.
\newblock Planbench: An extensible benchmark for evaluating large language models on planning and reasoning about change, 2023{\natexlab{b}}.

\bibitem[Valmeekam et~al.(2023{\natexlab{c}})Valmeekam, Marquez, Sreedharan, and Kambhampati]{planning23valmeekam}
Valmeekam, K., Marquez, M., Sreedharan, S., and Kambhampati, S.
\newblock On the planning abilities of large language models - {A} critical investigation.
\newblock In \emph{Advances in Neural Information Processing Systems 36: Annual Conference on Neural Information Processing Systems 2023, NeurIPS 2023, New Orleans, LA, USA, December 10 - 16, 2023}, 2023{\natexlab{c}}.

\bibitem[Vaswani et~al.(2017)Vaswani, Shazeer, Parmar, Uszkoreit, Jones, Gomez, Kaiser, and Polosukhin]{vaswani17attention}
Vaswani, A., Shazeer, N., Parmar, N., Uszkoreit, J., Jones, L., Gomez, A.~N., Kaiser, L., and Polosukhin, I.
\newblock Attention is all you need.
\newblock In \emph{Advances in Neural Information Processing Systems 30: Annual Conference on Neural Information Processing Systems 2017, December 4-9, 2017, Long Beach, CA, {USA}}, pp.\  5998--6008, 2017.

\bibitem[Wei et~al.(2022)Wei, Wang, Schuurmans, Bosma, Ichter, Xia, Chi, Le, and Zhou]{wei22cot}
Wei, J., Wang, X., Schuurmans, D., Bosma, M., Ichter, B., Xia, F., Chi, E.~H., Le, Q.~V., and Zhou, D.
\newblock Chain-of-thought prompting elicits reasoning in large language models.
\newblock In \emph{Advances in Neural Information Processing Systems 35: Annual Conference on Neural Information Processing Systems 2022, NeurIPS 2022}, 2022.

\bibitem[Welleck et~al.(2020)Welleck, Kulikov, Kim, Pang, and Cho]{welleck2020consistency}
Welleck, S., Kulikov, I., Kim, J., Pang, R.~Y., and Cho, K.
\newblock Consistency of a recurrent language model with respect to incomplete decoding.
\newblock In \emph{Proceedings of the 2020 Conference on Empirical Methods in Natural Language Processing (EMNLP)}, November 2020.

\bibitem[Wies et~al.(2023)Wies, Levine, and Shashua]{weis23subtask}
Wies, N., Levine, Y., and Shashua, A.
\newblock Sub-task decomposition enables learning in sequence to sequence tasks.
\newblock In \emph{The Eleventh International Conference on Learning Representations, {ICLR} 2023}, 2023.

\bibitem[Williams \& Zipser(1989)Williams and Zipser]{williams89tf}
Williams, R.~J. and Zipser, D.
\newblock A learning algorithm for continually running fully recurrent neural networks.
\newblock \emph{Neural Computation}, 1\penalty0 (2):\penalty0 270--280, 1989.

\bibitem[Wu et~al.(2016)Wu, Schuster, Chen, Le, Norouzi, Macherey, Krikun, Cao, Gao, Macherey, et~al.]{wu16nmt}
Wu, Y., Schuster, M., Chen, Z., Le, Q.~V., Norouzi, M., Macherey, W., Krikun, M., Cao, Y., Gao, Q., Macherey, K., et~al.
\newblock Google's neural machine translation system: Bridging the gap between human and machine translation.
\newblock \emph{arXiv preprint arXiv:1609.08144}, 2016.

\bibitem[Xue et~al.(2023)Xue, Likhosherstov, Arnab, Houlsby, Dehghani, and You]{xue2023adaptive}
Xue, F., Likhosherstov, V., Arnab, A., Houlsby, N., Dehghani, M., and You, Y.
\newblock Adaptive computation with elastic input sequence.
\newblock In \emph{International Conference on Machine Learning, {ICML} 2023}, Proceedings of Machine Learning Research. {PMLR}, 2023.

\bibitem[Yao et~al.(2023{\natexlab{a}})Yao, Yu, Zhao, Shafran, Griffiths, Cao, and Narasimhan]{yao23tree}
Yao, S., Yu, D., Zhao, J., Shafran, I., Griffiths, T., Cao, Y., and Narasimhan, K.
\newblock Tree of thoughts: Deliberate problem solving with large language models.
\newblock \emph{Advances in Neural Information Processing Systems}, 36, 2023{\natexlab{a}}.

\bibitem[Yao et~al.(2023{\natexlab{b}})Yao, Zhao, Yu, Du, Shafran, Narasimhan, and Cao]{yao23react}
Yao, S., Zhao, J., Yu, D., Du, N., Shafran, I., Narasimhan, K.~R., and Cao, Y.
\newblock React: Synergizing reasoning and acting in language models.
\newblock In \emph{The Eleventh International Conference on Learning Representations, {ICLR} 2023}, 2023{\natexlab{b}}.

\bibitem[Young \& You(2022)Young and You]{inconsistencies23young}
Young, T. and You, Y.
\newblock On the inconsistencies of conditionals learned by masked language models.
\newblock \emph{arXiv preprint arXiv:2301.00068}, 2022.

\bibitem[Zelikman et~al.(2022)Zelikman, Wu, Mu, and Goodman]{zelikman22star}
Zelikman, E., Wu, Y., Mu, J., and Goodman, N.~D.
\newblock Star: Bootstrapping reasoning with reasoning.
\newblock In \emph{Advances in Neural Information Processing Systems 35: Annual Conference on Neural Information Processing Systems 2022, NeurIPS 2022, New Orleans, LA, USA, November 28 - December 9, 2022}, 2022.

\bibitem[Zhang et~al.(2023)Zhang, Li, Meng, Chang, and den Broeck]{zhang23paradox}
Zhang, H., Li, L.~H., Meng, T., Chang, K., and den Broeck, G.~V.
\newblock On the paradox of learning to reason from data.
\newblock In \emph{Proceedings of the Thirty-Second International Joint Conference on Artificial Intelligence, {IJCAI} 2023, 19th-25th August 2023, Macao, SAR, China}, pp.\  3365--3373. ijcai.org, 2023.

\bibitem[Zhao et~al.(2023)Zhao, Kumar, Levine, and Finn]{zhao23act}
Zhao, T.~Z., Kumar, V., Levine, S., and Finn, C.
\newblock Learning fine-grained bimanual manipulation with low-cost hardware.
\newblock In \emph{Robotics: Science and Systems XIX, Daegu, Republic of Korea, July 10-14, 2023}, 2023.

\bibitem[Ziegler et~al.(2019)Ziegler, Stiennon, Wu, Brown, Radford, Amodei, Christiano, and Irving]{ziegler19rl}
Ziegler, D.~M., Stiennon, N., Wu, J., Brown, T.~B., Radford, A., Amodei, D., Christiano, P., and Irving, G.
\newblock Fine-tuning language models from human preferences.
\newblock \emph{arXiv preprint arXiv:1909.08593}, 2019.

\end{thebibliography}
\bibliographystyle{icml2024}

\newpage
\appendix

\onecolumn

\section{Limitations}
\label{sec:limitations}
\begin{enumerate}
    \item Our arguments are empirical and conceptual. We have not provided
a formal proof for our arguments. %
    \item We have also not demonstrated failure for very large models such as \texttt{Llama2} \citep{touvron2023llama} or \texttt{Mistral} \citep{jiang2023mistral}.  
    \item We note that there may be specific workarounds to make the Transformer (efficiently) learn the  path-star task. For example, it may become solvable with other pre-trained models which may have been taught path-finding with step-by-step supervision (see Remark~\ref{rem:indecipherable}). The task may also be solvable via other workarounds such as in-context learning or multi-modal learning (where the model visually processes the image of the graph).  
    Furthermore, since the path-star problem composes the same subroutine over itself, it allows efficient parallelizable solutions \citet{sanford2024transformers} that may not require learning a sequential composition of discrete subroutines. This can break the assumption of Proposition~\ref{prop:indecipherable}, making the optimization tractable.  We warn the reader that these are however only specialized workarounds for this specific task; our broader point is that (a) next-token-prediction/teacher-forcing may still be relatively inefficient compared to teacherless training and (b) there may still be other novel tasks not seen during pre-training, or not solvable visually, or where no paralleizable solutions exist, for which similar failures may occur. We discuss this in detail in Remark~\ref{rem:indecipherable}.
    \item Nevertheless, beyond the minimal path-finding setting, we have not demonstrated or characterized the range of problems where teacher-forcing-induced failure may occur. We only intuitively believe it should extend to other problem-solving tasks and creative-writing tasks that require lookahead (see \S~\ref{sec:story}).
    \item  It is also  unclear if this failure generalizes to run-of-the-mill text-generation tasks. %
    
\end{enumerate}

\section{Teacher-Forcing Failure and Snowballing Failure are Distinct}
\label{sec:diff}

We emphasize that, while both the Clever Hans failure mode and the Snowball mode are both indicative of the inability to plan, these failure modes are also orthogonal to each other, and demand different solutions. We make this a bit more formal:

\begin{proposition}
In the path finding problem of \S\ref{sec:minimal-task}, there exists a next-token predictor that experiences Failures~\ref{fail:clever},~\ref{fail:diff-token} due to teacher-forcing, but not the snowballing error Failure~\ref{fail:snowball} due to autoregressive inference. Conversely, there exists a next-token predictor that experiences the latter failure but not the former. 
\end{proposition}

\begin{proof}
Consider the model learned via teacher-forcing on the graph problem. During inference, we saw that it suffers a debilitating error right in the first step (with accuracy of $1/d$ for degree $d$ of the start node). Thus, during inference the error that is experienced is not from an accumulation over length. In fact, if only the first node is set correctly during inference, a model with the perfect Clever Hans cheat, would achieve $100\%$ accuracy rate. Such a model does not experience snowballing errors. 

On the other hand, consider a model, that in each step predicts the correct next vertex with a high accuracy of $1-\epsilon$ for small $\epsilon$. Such a model clearly has learned the correct plan, albeit with minor errors in each token. These errors however can snowball during inference. Thus, this model has no failure due to teacher-forcing, but will fail during autoregressive inference, if the path length is long. 
\end{proof}

\textbf{Differing solutions.} Based on the above simple illustration, we note that the two failures need different solution approaches. Specifically, while snowballing errors may be fixable via ``backtracking-and-planning'' wrappers, teacher-forcing failures is a pathology that cannot be solved post-hoc.

\section{An Illustration via Story-Telling}
\label{sec:story}
 Can a teacher-forced model merely trained on thousands of stories learn to write plot twists? Indeed, \citet{bubeck23sparks} report instances where models can fail to accomplish tasks involving creative-writing (e.g., poems). We speculatively extend our discussion in \S\ref{sec:cleverhans} to reason about this scenario. Consider for example, teacher-forcing on the following story that follows an often-used plot outline:

{ \small 
\begin{itemize}
    \item \texttt{Event 1 (Setup):} Alex and Bob, who are friends, are trying to defeat the Evil King.
    \item \texttt{Event 2 (Conflict):} One day, surprisingly, Bob  turns against Alex, and tries to thwart Alex's plans, {\em albeit unsuccessfully}.
    \item \texttt{Event 3:} Alex thinks Bob is evil too, defeats Bob first.
    \item \texttt{Event 4 (Backstory):} Losing the battle, Bob reveals he is a double-agent. In his final words, Bob explains he was ordered to defeat Alex.
    \item \texttt{Event 5 (Resolution): }  To preserve the King's trust, Bob obeyed the command, but also \textit{deliberately} failed at it. Bob then relays critical information he extracted from the King's inner circles. 
    \item \texttt{Event 6:} Alex uses Bob's insider information to defeat the King.  
\end{itemize}
}

Evidently, this story requires a plan: \texttt{Event 5} is a key plot resolution that the narrator must have planned before methodically generating parts of the setup in \texttt{Event 1} (introducing Bob as a friend) and the conflict in \texttt{Event 2} (Bob's turning against Alex, and failing at it).  While training, the model must thus treat the story as a whole, and tease apart these dependencies between the events, some of which may be anti-chronological (akin to how, in the path-star graph, the model must learn that the problem is straightforwardly solvable when viewed from right-to-left). 

However, we hypothesize that a teacher-forced model would take a rigid chronological (left-to-right) view. First, it would use the Clever Hans cheat to easily fit the plot resolution in \texttt{Event 5}: the model would use the facts of \texttt{Event 4} and \texttt{Event 2} (revealed as input) to fit the content of Bob's final words. Thus, the content of \texttt{Event 5} would no longer be available as supervision to guide how the model fits \texttt{Event 1} and \texttt{Event 2}. When the model tries to fit these earlier events, these events would become Indecipherable Tokens --- the model would simply learn to fit them as arbitrary events. Thus, we conjecture that a model trained via teacher-forcing merely on raw, unannotated texts of stories --- however many stories they may be --- would not learn to plan its stories, and would instead create arbitrary twists and turns during inference, and improvize upon that.

\section{Other Remarks}

\begin{remark} \label{rem:indecipherable} \textbf{(Conditions under which first token becomes decipherable end-to-end)} There are certain corner cases where teacher-forcing can learn the (otherwise indecipherable) first token efficiently, without having to brute-force search an exponential space of algorithms. We enumerate these below. (Note though that  even if the problem does become tractable, it can still be the case in these problems, next-token prediction is less efficient than predicting multiple future tokens.)
\begin{enumerate}
    \item \textbf{Lucky prior biases:}  If the model happens to have been exposed to certain relevant kinds of supervision during pre-training, then the model will be biased towards a favorable part of the search space, and  chance upon the right algorithm much quicker. 
    \begin{enumerate}
        \item  If the model had witnessed the same task but with the correct step-by-step supervision, then the prior would assign high probability to the correct algorithm. (One can imagine that the the true end-to-end algorithm itself becomes a readily available subroutine in this case.)%
        \item Or, in the specific path-star example, if the prior bias assigns high probability to all $l$ subroutines being identical, then the search only needs $O(|\mathcal{C}|)$ time (where $\mathcal{C}$ is the set of candidate subroutines.) For illustration, see the Transformer construction for k-hop problems in \cite{sanford2024transformers}.
    \end{enumerate}
    Note that in such a case, one may still be able to demonstrate intractability by constructing slight variations of the tasks that defy such prior biases (e.g., tasks where the subroutines are not identical).
    \item \textbf{Small graph size.} If the number of edges is very small (say $|E|$) in proportion to the training data, then the model can learn alternative solutions that ``memorize'' the problem:
        \begin{enumerate}
            \item \textbf{Naive memorization:} If the vocabulary has only $|\vocab|$ possible node ID's, then there are only $O(|\vocab|^{|E|})$ possible adjacency lists the model can see. So, if the training data has at least $\Omega(|\vocab|^{|E|})$ datapoints, then, every test example is seen with high probability during training. Here, the model merely needs to look up exact replicas from training, and regurgitate the subsequent values from the training string.
            \item \textbf{Node-ID-agnostic memorization:} If the number of training data is $\Omega(|E|!)$, the model can still implement a  form of memorization, but this requires a cleverer strategy that is agnostic to the node ID's. Concretely, for a fixed assignment of node ID's, there are only $O(|E|!)$ ways to permute the adjacency list. Assume that the model sees all such permutations during training (albeit with varying instantiations of the node IDs). During inference, the model can first look up whether the test input corresponds to an existing permutation it had witnessed (possibly with different node ID's). For example $1 \to 2; 2 \to 3; 4 \to 1$ would correspond to $10 \to 20; 20 \to 30; 40 \to 10$. Then the model merely needs to recall from its memory, the index where the target token is located in the input for that permutation. The model can then look at the same index in the test input and output the node ID located there. Thus, if the target token was $4$ for the first sequence above, the model can output $40$ for the other sequence.
        \end{enumerate}
    Note that this doesn't contradict Proposition~\ref{prop:indecipherable} because the proposition only precludes learning the ``true'' path-finding algorithm, not spurious memorizing solutions.
    \item \textbf{Small path length.}  If the path length $l$ is small, then either the search space of algorithms (which is about as large as $|\mathcal{C}|^l$) becomes tractable. 
        \begin{enumerate}
            \item For example, if $l=3$, the first node is the only intermediate node between the start and goal token. Here, the model can easily learn the ``right-to-left'' solution that the desired node is the only node preceding the goal node. 
        \end{enumerate}
\end{enumerate}

\end{remark}

\begin{remark} \label{rem:teacherles} \textbf{(Mechanism implemented by teacherless model)}
The (hypothetical) solution that the teacherless model must implement is a fairly difficult one to implement --- yet the model surprisingly learns to implement it. Recall that our hypothesis is that the teacherless model automatically learns to fit the targets in the reverse order, since the path \textit{from} $\goal$ is unique.  This is indeed what we find in Fig~\ref{fig:diff-traj}, where the accuracies of the later tokens become higher earlier. 
Note though that this is a fairly difficult computation to implement. First, when the model predicts $\node_{i}$, it must require the identity of $\node_{i+1}$. However, this identity is not fed as input to the model, in the absence of the teacher. Thus the model must have computed $\node_{i+1}$ and crucially, stored that in one of its its internal representations. Then, by induction, when predicting the first node $\node_{1}$, the model must know the identity of \textit{all} the other nodes in the path. In other words, the model must have (a) computed and (b) stored the whole solution in its hidden representations before it outputs the first token. This is a substantial type of lookahead that some of our models are able to achieve under teacherless training. 
\end{remark}

\section{Experiment Notations}

\subsection{Verifying In-Distribution Performance}
We simply compute exact match with the ground truth path as follows:
\begin{align}
    \agacc(\llm) := \mathbb{P}( \hat{\response} = \response)  \text{,} & & \prefix, \response \sim \distr,\hspace{2mm} \hat{\response} \agsample \llm. \label{eq:gen-acc} 
\end{align}

\subsection{Quantifying the Clever Hans Cheat}

 Formally, let $\texttt{Unif}(\mathcal{N}(\start))$ denote a uniform distribution over the set of adjacent nodes of $\start$. For any node $\node$ in the graph, denote by $\outpath(\node)$ the path emanating from $\node$ and going outwards, away from the start node. Notice that except for $\node = \start$, this path is unique.  We thus measure 
\begin{align*}
&
    \cheatacc(\llm) \coloneq \mathbb{P}\left( \hat{\path}_{1<}= \outpath(\node_1') \right) \numberthis  \label{eq:cheatacc}
\\[2mm] 
&\hspace{2mm}\text{where}\hspace{3mm}
    \prefix, \response \sim \distr, 
\hspace{2mm}
    \hat{\path}_{1<} \agsample \llm(\cdot ; \prefix, \start, \node_1' )  
\\[1mm] 
&\hspace{14mm}
    \node_1' \sim \texttt{Unif}(\mathcal{N}(\start)). 
\end{align*}

\subsection{Quantifying the Indecipherable Token Failure}

To quantify the Indecipherable Token failure, in our path-finding task, we measure the accuracy in predicting the first token after the start node.

\begin{align}
\firstacc(\llm) =
\mathbb{P}\left(\hat{\responsetoken}_1 = \responsetoken_1 \right), && \prefix, \response \sim \mathcal{D}, \hat{\response} \agsample \llm(\cdot;\prefix). \label{eq:first-acc}
\end{align}

\subsection{Inference in Teacherless Training}

In teacherless training, we make use of an uninformative input $\response^{\$}$ that simply corresponds to a series of dummy tokens denotes by $\$$. During inference, instead of autoregressing on the model's own output, we use this uninformative input. We formalize this below:

\begin{align}
    \hat{\response} \foresightsample \llm^{\$}(\cdot;\prefix) \text{ where } \hat{\responsetoken}_i \sim \llm^{\$}(\cdot;\prefix,  \response^{\$}_{<i}).
\end{align}

We then denote the accuracy of the model as follows:
\begin{align}
    \dollaracc(\llm^{\$}) = \mathbb{P}\left(\hat{\response}=\response \right)   &&  \prefix, \response \sim \mathcal{D}, \hspace{2mm}\hat{\response} \foresightsample \llm^{\$}(\cdot;\prefix). \label{eq:foresight-acc}
\end{align}

\subsection{Reversed Training}

Notationally, in reversed training we let $\llm^{\texttt{rev}}$ be the model trained to maximize $\ntpobj$ with the targets (and the teacher-forced inputs) set to  $\response^{\texttt{rev}}=\responsetoken_{\responselen}, \hdots \responsetoken_1$, the reversal of $\response$. We then measure the autoregressive accuracy by comparing against $\response^{\texttt{rev}}$:
\begin{align}
    \reverseacc(\llm^{\texttt{rev}}) = \mathbb{P}\left(\hat{\response}=\response^{\texttt{rev}}\right), && \; \prefix, \response \sim \mathcal{D}, \hat{\response} \agsample \llm^{\texttt{rev}}(\cdot;\prefix) \label{eq:rev-acc} 
\end{align}
\section{More Experimental Results}
\subsection{Snowball Failure}
\label{sec:snowball-expt}
To explicitly measure to what degree the model falls victim to the snowball effect, we train \textit{GPT-Mini} on graphs of various path lengths $l$. In order to remove the failure stemming from the difficult first token, we teacher-force the model for the first token and then check how accurate the generations are for subsequent tokens. More concretely, we evaluate
\begin{align*}
    &\snowballacc(\llm) = \mathbb{P}\left(\hat{\response}_{1<} = \response_{1<}\right) \numberthis \\[2mm]
    &\hspace{-5mm}\text{where }\hspace{2mm}  \prefix, \response \sim \mathcal{D}, \hspace{2mm}\hat{\response}_{1<} \agsample \llm(\cdot;\prefix, \responsetoken_1)  \nonumber
\end{align*}
If $\snowballacc(\llm)$ is $\approx 1$, then \textit{Failure \ref{fail:snowball}} is not prominent in our task. If $\snowballacc(\llm) \ll 1$, then clearly teacher-forcing is responsible for surpressing errors in generation, strongly hinting at the fact that \textit{Failure \ref{fail:snowball}} is at play. We display the results in Fig.~\ref{fig:snowball} (left). We observe that the accuracy $\snowballacc$ is barely affected even for graphs with very long paths $L=40$. 

As another metric, we proceed token by token during inference, and evaluate the probability of correctly predicting all tokens up to the current one. We report this for $G_{2, 40}$ in Fig.~\ref{fig:snowball} (right).  Similarly, while the success probability does decay for larger length (at an exponential rate), it remains very high due to the failure events being so unlikely. We thus conclude that \textit{Failure \ref{fail:snowball}} is not as prominent in this setting.

\begin{figure}[H]
    \centering
    \includegraphics[width=0.45\textwidth]{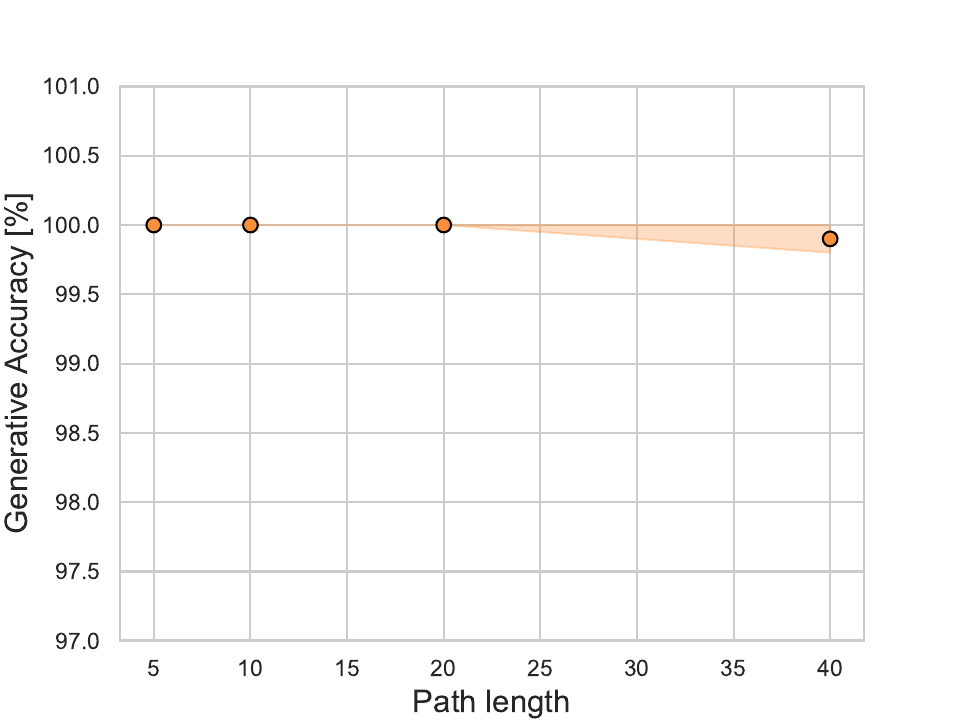}%
    \includegraphics[width=0.5\textwidth]{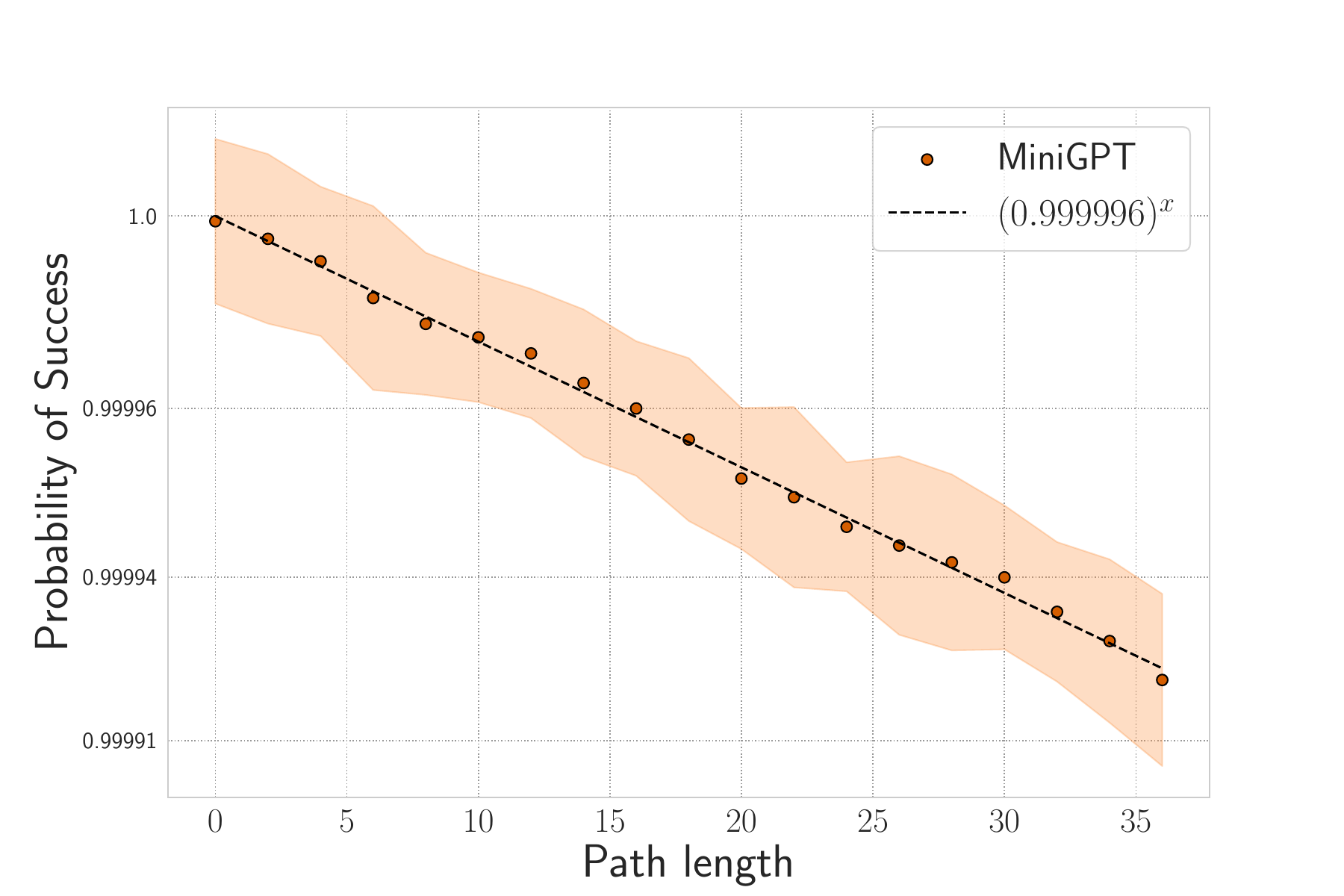}
    \caption{Accuracy of $\llm$ when conditioned on the first difficult token (left) for graphs of various length. Probability of correct prediction of $\llm$ as a function of current token position on $G_{2, 40}$, as we walk towards the goal.}
    \label{fig:snowball}
\end{figure}
\label{sec:more-results}
\subsection{Clever Hans Cheating Accuracies}
In Table~\ref{table:cheat-accs} we display the Clever Hans cheating accuracies $\cheatacc(\llm)$. We observe that in almost all cases, all the models achieve nearly perfect cheating accuracies. The only exception is the high-degree graph $G_{20,5}$ where all models struggle to even fit the training data.
\begin{table}[H]
\vskip 0.15in
\begin{center}
\begin{small}
\begin{sc}
\begin{tabular}{lccccc}
\toprule
  & $G_{2, 5}$ & $G_{2, 20}$ & $G_{5, 5}$ & $G_{10, 5}$ & $G_{20, 5}$ \\
\midrule
 GPT-Mini  & $99.7$ & $100$& $100$ & $99.8$ & $0.0$\\
 \midrule 
 GPT2-Large & $99.8$ & $99.7$ & $100$ & $99.8$ & $0.0$\\
 \midrule
 Mamba & $97.6$ & $98.3$ & $99.5$ & $95.9$ & $0.0$ \\
\bottomrule
\end{tabular}
\end{sc}
\end{small}
\end{center}
\caption{Evaluating Clever Hans cheating  accuracies $\cheatacc(\llm)$ (in percent $\%$) for different types of graphs.}
\label{table:cheat-accs}
\end{table}
\subsection{More Detailed Accuracies}
We report more detailed accuracy values per model in the following tables. We display standard accuracy $\agacc(\llm)$ in Table.~\ref{table:gen-acc}, teacherless accuracy $\dollaracc(\llm)$ in Table.~\ref{table:dollar-accs} and reverse accuracy $\reverseacc(\llm)$ in Table.~\ref{table:reverse-gen}. In general we observe that solving the task with standard next-token prediction is very tough and performance is limited to $\frac{1}{d}$ where $d$ is the degree of the graph $G_{d,l}$. 

\begin{table}[H]
\vskip 0.15in
\begin{center}
\begin{small}
\begin{sc}
\begin{tabular}{lccccc}
\toprule
  & $G_{2, 5}$ & $G_{2, 20}$ & $G_{5, 5}$ & $G_{10, 5}$ & $G_{20, 5}$ \\
\midrule
 GPT-Mini  & $49.8$ & $49.1$& $19.1$ & $8.1$ & $0.0 $\\
 \midrule 
 GPT2-Large & $48.9$ & $49.2$ & $19.4$ &$10.3$ & $3.5$\\
 \midrule
 Mamba & $48.5$ & $48.7$ & $20.2$ & $9.3$ & $0.0$\\
\bottomrule
\end{tabular}
\end{sc}
\end{small}
\end{center}
\caption{Autoregressive accuracies $\agacc(\llm)$ (in percent $\%$) for different types of graphs.}
\label{table:gen-acc}
\end{table}

Teacherless training on the other hand works very well with GPT2-Large, allowing it to solve most graph tasks perfectly. From-scratch models however also struggle to learn the task in this fashion (except for GPT-Mini on the simplest graph, $G_{2,5}$). 
\begin{table}[H]
\vskip 0.15in
\begin{center}
\begin{small}
\begin{sc}
\begin{tabular}{lcccccc}
\toprule
  & $G_{2, 5}$ & $G_{2, 10}$& $G_{2, 20}$ & $G_{5, 5}$ & $G_{10, 5}$ & $G_{20, 5}$ \\
\midrule
 GPT-Mini  & $99.9$ & $0.0$ & $0.0$& $0.0$ & $0.0$ & $0.0$\\
 \midrule 
 GPT2-L &$99.9$ & $98.8$ & $0.0$ & $99.0$ & $97.8$ & $0.0$ \\
 \midrule
 Mamba & $0.0$ & $0.0$ & $0.0$ & $0.0$ & $0.0$ & $0.0$ \\
\bottomrule
\end{tabular}
\end{sc}
\end{small}
\end{center}
\caption{Autoregressive accuracy $\dollaracc$ when using a teacherless response.} 
\label{table:dollar-accs}
\end{table}

Finally, reversing the sequence significantly simplifies the problem for all the models, allowing near perfect accuracies across all graphs. 

\begin{table}[H]
\vskip 0.15in
\begin{center}
\begin{small}
\begin{sc}
\begin{tabular}{lccccc}
\toprule
  & $G_{2, 5}$ & $G_{2, 20}$ & $G_{5, 5}$ & $G_{10, 5}$ & $G_{20, 5}$ \\
\midrule
 GPT-Mini  & $99.7$ & $99.8$& $100$ & $99.8$ & $0.0$\\
 \midrule 
 GPT2-Large & $99.9$ & $99.9$  & $99.6$ &  $99.8$ & $99.9$ \\
 \midrule
 Mamba & $98.5$ &$96.2$ & $99.1$ & $99.5$ & $0.0$ \\
\bottomrule
\end{tabular}
\end{sc}
\end{small}
\end{center}
\caption{Autoregressive accuracy $\reverseacc$ when reversing the response $\bm{r}$.}
\label{table:reverse-gen}
\end{table}

\subsection{Arithmetic Tasks}
\label{sec:arithmetic}
To further highlight that the identified failure modes are relevant for realistic tasks, we study the task of addition. We consider 3-digit addition, where samples are of the form
$$x_1x_2x_3 + y_1y_2y_3 = z_1z_2z_3z_4$$
We use the natural encoding (i.e. use the digits themselves as encodings) and pad shorter numbers with leading zeros to ensure fixed lengths. It has been observed in prior work \citet{lee23teaching,shen23positional} that reversing the result (i.e. $z_4z_3z_2z_1$) is very beneficial for training, leading to significantly more sample-efficient learning. Here we study if our teacherless training strategy can lead to similar gains over the standard encoding, i.e. we encode the inputs as 
$$x_1x_2x_3 + y_1y_2y_3 = \$\$\$\$$$
while keeping all other aspects of the training pipeline (such as the targets and the objective) the same as in standard training. We again train a \textit{GPT-Mini}-style transformer with \textit{AdamW} using a learning rate of $5$e-$4$. We display the resulting test accuracies for both standard, reversed and teacherless training in Fig.~\ref{fig:addition}, as a function of training iterations. Here we follow previous works and at every iteration, sample with replacement from all the possible $3$-digit configurations (aside from a test set put aside before). We can indeed see that teacherless training leads to more sample-efficient learning compared to the standard format, but to slightly less efficient learning compared to the reverse format. This again hints at the fact that some form of CleverHans cheating is picked up when learning addition. 

In this case though, it is more difficult to precisely characterize the form of the cheat in contrast to the path-finding problem. One possible form of cheating may be that in addition, knowing the leading digits of the answer can provide information about the subsequent digits. For example, when we add two single-digit numbers (picked uniformly at random), knowing that the $10'th$ place is a $1$ rather than a $0$, tells us that the $0$th place is more likely to be a smaller digit like $0$. 

\begin{figure}
    \centering
    \includegraphics[width=0.5\textwidth]{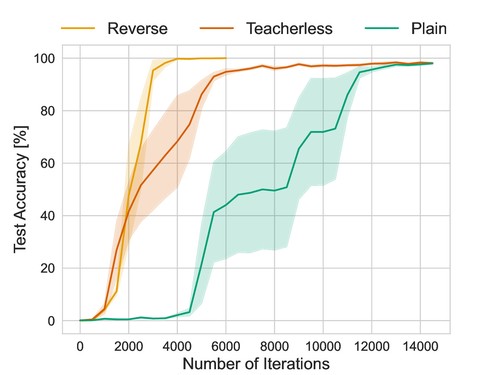}
    \caption{Test accuracies for 3-digit addition for standard, reversed and teacherless training.}
    \label{fig:addition}
\end{figure}

\subsection{Predicting Only the Indecipherable Token}

To study the difficulty of the first token in isolation from the Clever Hans cheat, we reduce the task to solely predicting this token, removing the rest of the path. The problem thus effectively becomes a classification problem. We use the same learning setup as for previous experiments and keep the same number of examples. We display the results in Table~\ref{table:only-first-token}. Our results indicate that the model fails to learn the first token in isolation.

This reinforces our understanding of the role played by the Clever Hans cheat in bringing about failure. One possibility could have been that the first node in isolation is easy to learn, but the Clever Hans cheat makes it harder somehow (e.g., it attracts the model to a bad local minimum). The other possibility (the one we put forth) is that the Clever Hans cheat reduces the problem to learning the first token in isolation, which we argued is hard. These experiments confirm the second possibility.

\begin{table}[H]
\vskip 0.15in
\begin{center}
\begin{small}
\begin{sc}
\begin{tabular}{lccccc}
\toprule
  & $G_{2, 5}$ & $G_{2, 20}$ & $G_{5, 5}$ & $G_{10, 5}$ & $G_{20, 5}$ \\
 \midrule 
 GPT2-Large & $50.2$ & $50.4$  & $18.9$ &  $10.4$ & $4.5$ \\
\bottomrule
\end{tabular}
\end{sc}
\end{small}
\end{center}
\caption{Accuracy when solely learning to predict the difficult token $r_1$}
\label{table:only-first-token}
\end{table}
\section{Other experimental details}

\subsection{Tokenization}
\label{sec:tokenizer}

We tokenize the graph in the following manner: (1) we first tokenize the randomly shuffled edge list as $``|v_1\hspace{1mm} v_2|v_3\hspace{1mm}v_4|..."$ where the first vertex in each edge is the one closest to $\start$, (2) then append start and goal node as $``/ \start\hspace{1mm}\goal=" $ and (3) then append the full path repeating start and goal node, $"\start\hspace{1mm}v_{i_1} \dots v_{i_{l-1}} \hspace{1mm} \goal"$. Note that (1) and (2) make up the prefix $\prefix$, which the model does not learn to predict. Then, (3) is the target sequence that the model aims to learn. The vocabulary size is thus given by $N+3$, where we add entries for the special tokens $``|"$, $``/"$ and $``="$. When using the pre-trained models GPT2 we use the tokenizer that was employed for pre-training, in this case the \textit{Byte-Pair tokenizer} \citep{Radford2019LanguageMA}.

\subsection{Models}

\label{sec:models}

 When training Transformer models from scratch, we use a small model consisting of $n_{\text{layers}}=12$ blocks with embedding dimension $e_{\text{dim}}=384$, $n_{\text{heads}}=6$ attention heads and MLP expansion factor $e=4$, coined \textit{GPT-Mini}. For pre-trained models, we consider GPT2-Large with $n_{\text{layers}}=36$, $e_{\text{dim}}=1280$, $n_{\text{heads}}=20$ and expansion factor $e=4$ \citep{Radford2019LanguageMA}. To further evaluate purely recurrent models, we perform experiments with the recent Mamba model \citep{gu2023mamba}. We train the Mamba models from scratch with $12$ layers and embedding dimension $784$. We train all the models with the \textit{AdamW} optimizer \citep{loshchilov2018decoupled}. For models trained from scratch we use a learning rate of $\eta =0.0005$ while for pre-trained models we use a smaller one of $\eta =0.0001$. In both cases we use weight decay of strength $0.01$. Models from scratch are trained for up to $500$ epochs in order to ensure convergence. Pre-trained models require less training time and we usually fit the training data perfectly after $10$ epochs.

\section{More Related Work}
\label{sec:more-related-work}

\textbf{Other arguments about next-token prediction.} We note that the works of \citet{kaariainen2006lower,ross10efficient}
 capture a stronger notion of snowballing, wherein, once an erroneuous sub-optimal action is committed, the model is more likely to commit more sub-optimal actions since it has wandered into territories that it was not trained on. Implicitly, the error here is not evaluated as an exact match of the response (i.e., $\response \neq \hat{\response}$) but as a cumulative notion of error over all steps (e.g., $\sum\mathbf{1}[\responsetoken_i \neq \hat{\responsetoken}_i]$). In this setting, there is an additional cause of failure called \textit{exposure bias}: the teacher-forced model has only been trained on correct trajectories, and has not learned how to {recover} from poor trajectories. Nevertheless, even this notion of snowballing assumes that that teacher-forcing has learned an accurate next-token predictor in the first place, which our failure mode challenges.

A closely-related criticism \citep{bubeck23sparks,dawid23latent,lecun24ar,du23lookahead} is that to model human thinking, we need to model two types of thinking as outlined in \citet{kahneman2011thinking}: a fast (System 1) thinking process that is also guided by a slower (System 2) thinking process. 
Theoretically, \citet{lin21limitations} show that there are formal languages for which expressing some next-tokens may require super-polynomial time or parameter count during \textit{inference}. These arguments however only suggest that some tokens require more computation; not that they are specifically problematic under left-to-right learning. However, \citet{du23lookahead} informally note that some next tokens can be hard to \textit{learn} as they require a global understanding of what will be uttered in the future (but see Remark~\ref{rem:local-unlearnable} below).

\begin{remark} 
\label{rem:local-unlearnable} \textbf{(Locally Unlearnable Token vs Indecipherable Token)}
We note that the ``locally unlearnable'' hypothesis of  \citet{du23lookahead} is related to, but not the same as the Indecipherable Token failure. The hypothesis in \citet{du23lookahead} is that when we learn tokens left-to-right, some tokens simply cannot be learned since crucial information becomes available only in subsequent tokens. This hypothesized failure may happen regardless of whether the supervision from subsequent tokens is lost to a Clever Hans cheat.   In contrast, in our path-star graph, the Indecipherable Token becomes unlearnable only because of the Clever Hans cheat. For example, 
 the (first) Indecipherable Token in the path-star problem \textit{is} locally learnable by the teacherless model (where, the crucial information is still only presented after this token). This token becomes unlearnable only  in the teacher-forced model where the Clever Hans cheat emerges. 
 \end{remark}

 We survey related arguments of next-token prediction, orthogonal to our main discussion regarding planning. \citet{zhu23physics,lv23reversal} report that language models that are trained on \texttt{A equals B} are unable to infer \texttt{B equals A}, which \citet{zhu23physics} suggest is due to autoregressive left-right training. \citet{du2023measure,welleck2020consistency} formalize the limitation that autoregressive models may potentially assign non-zero probability to infinite-length strings, thus leading to non-terminating inference. \citet{li24ntp} provide a Transformer-specific analysis of how self-attention affects the optimization geometry of next-token prediction. \citet{thrampoulidis2024implicit} provide an analysis of the implict bias of optimization with next-token prediction for linear models.

\textbf{Other limitations of Transformers} \citet{merrill2023parallelism} identify limitations of the representative power of Transformer architecture when the arithmetic precision is logarithmic in the number of input tokens.
\citet{bender21stochastic} criticize GPT-like language models as simply parroting out training data with minor stochasticity, while \citet{arkoudas2023chatgpt} report that such models struggle with reasoning, even if not a stochastic parrot.  \citet{inconsistencies23young} study masked language (T5, BERT) models (not causally-trained) and argue there are inconsistencies in the probabilities that they assign. E.g., when conditioned on `\texttt{white}', the probability of `\texttt{rice}' may be higher `\texttt{bread}' but the probability of `\texttt{white bread}' and `\texttt{white rice}' are the opposite. \citet{artetxe22role} empirically analyze the effect of bidirectional attention and bidirectional supervision (as in masked language modeling) during pretraining on the ability of the model to do various things, including next-token prediction. \citet{springer2024repetition} argue that autoregressive Transformers compute sub-optimal embeddings that can be improved by repeating the input text twice.

Finally, we note that \citep{ranaldi2023hans} use the term Clever Hans effect to denote how models can pick up spurious correlations between the position of a choice in a multiple-choice question, and the correctness of the answer. We note that the above correlation is inherent to the distribution, and independent of teacher-forcing. We distinguish this from the Clever Hans \textit{cheating} which happens under the guidance of teacher-forcing.

\textbf{End-to-end reasoning and chain-of-thought supervision. } 
In our path-star graph, learning the Indecipherable Token (the first node $\node_1$) can be thought of as a task whose end target is $\node_1$, but whose implicit intermediate targets (or ``chain-of-thought'') correspond to the unique path starting from $\goal$ headed towards $\node_1$ (although this is only provided as supervision after the first token). In this terminology, we can rephrase our claim as the model failing to learn the end target once the intermediate targets are lost to the Clever Hans Cheat. 

Such limits of end-to-end learning have been echoed in literature on learning with chain-of-thought-type supervision. Recent theoretical works have shown broad classes of tasks (e.g., any function efficiently computed by a Turing machine) where prepending CoT to the end target allows efficiently learning tasks; yet, there are ``multi-hop reasoning'' tasks that  are unlearnable end-to-end (i.e., without intermediate supervision) either due to computational hardness \citep{weis23subtask} or representational limits \citep{malach23auto}). Earlier theoretical works \citet{shwartz17failures,shwartz16endtoend} have similarly proven negative  results for end-to-end learning in similar settings. Similar empirical arguments have been made in neural network literature \citep{gulcehre16knowledge,glasmachers17endtoend} and also more recently, in language models
on complex reasoning and math problems \citep{nye21scratchpad,ling17induction,cobbe21training,
piekos20measuring,
zelikman22star,
recchia21teaching,
cobbe21training,
hsieh2023distilling,shridhar2022distilling}.

\begin{remark}\textbf{(Chain-of-thought before vs. after end target.)}
It is worth noting though that the above lines of work are concerned with chain-of-thought that is present before the end target; in our setup, this supervision is presented only \textit{after} the end target. Surprisingly, some of our teacherless models manage to utilize even such hindsight chain-of-thought.
 This success is not fully explained by existing positive results about chain-of-thought supervision, such as \citet{weis23subtask,malach23auto}, where supervision is provided \textit{before} the end target.
 \end{remark}

\textbf{Going beyond next-token prediction.} Inference-time  techniques like chain-of-thought \citep{reynolds2021prompt,wei22cot,kojima2022large} and its variants \citep{yao23tree,besta23graph,yao23react} or those that elicit feedback from the model \citep{madan23selfrefine,huang22internal,shinn2023reflexion} can be thought of as going beyond conventional form of inference by allowing the model to think more before producing its final answer. However, the backbone in these models are still trained by standard teacher-forcing. While other techniques \citep{burtsev2020memory,xue2023adaptive,goyal23think} train the model to explicitly think more, even these boil down to next-token prediction during training.

 One may argue that reinforcement learning-based training \citep{ranzato16seq,wu16nmt,bahdanau17actor,paulus18deep,ziegler19rl, stiennon20rl, ouyang19rl} is another way to build backbones that go beyond teacher-forcing. However, it is worth noting that the gradients in these techniques boil down to teacher-forcing on the model's own generated answer. Furthermore, if we desire that the model be able to generate a solution that can plan ahead of time, it is unclear how a model can go from a complete inability to plan (that may assign near-zero probability to the true plan in an exponential space of solutions), to discovering the correct plan simply through preference-based feedback (see \citep{havrilla2024teaching} for related empirical evidence). 

Another line of work --- spanning language \citep{bengio15scheduled,goyal16professor}, imitation learning \citep{ross11reduction,ross10efficient,ross14aggrevate} and structured prediction \citep{daume09structured,chang15lols} --- has been aimed at addressing the Snowball Failure, under the assumption that the model has otherwise learned an accurate next-step predictor. Broadly, the idea is to train the model on a mixture of the ground truth sequences and the model-generated sequences themselves, as a way to ensure that the test-time and training-time distributions are as similar as possible. These techniques however do not address the failure to learn a good next-step predictor in the first place.

As for reversal-based training, \citet{lee23teaching,shen23positional} observe that addition tasks become much simpler when the digits are reversed. Their argument is that this explicitly assists the model to learn a simpler algorithm. When it comes to natural language however, \citet{papadopoulos2024arrows} find that reversing hurts the model's perplexity.

\textbf{Predicting future tokens.} Some works \citep{gurnee10finding,meng22locating,pal23future} aim to recover future tokens  that an already-trained model may predict based on the internal layers of the current token. Note that the success of this does not imply that the model necessarily plans well. This only means that it is possible to recover what the already-trained model wants to generate in the future (which may simply be a bad plan). %
 \citet{pfau2023eliciting} train a language model to predict in reverse with the orthogonal goal of finding prefixes that elicit certain behaviors.  
 
 \textbf{Shortcut-learning in language models.} A line of work has empirically and theoretically analyzed how Transformer-based language models learn superficial shortcuts to (partially) solve tasks such as learning multiplication \citep{dziri23faith}, logic \citep{zhang23paradox}, automata \citep{liu23automata}, recursion \citep{inconsistencies23young}, reading comprehension \citep{mrc21lai} and multiple-choice questions \citep{ranaldi2023hans}
However, these shortcuts must \textit{not} be confused with the Clever Hans cheating induced by teacher-forcing as elaborated below.

\begin{remark} \label{rem:shortcuts} \textbf{(Difference between Clever Hans cheating and known shortcut-learning failures in Transformers.)} First, 
these aforementioned shortcuts exist independent of teacher-forcing:  these are correlations between the prefix (such as the initial digits of two multiplicands) and the final answer (the initial digits of the product) in the underlying training distribution. But Clever Hans cheats 
arise only upon teacher-forcing: 
 these are correlations between the prefixes of the answer itself to the rest of the answer. Second, the above shortcuts  only fail out-of-distribution (such as when the number of multiplied digits is increased, where the failure is in length generalization \citep{anil22length}).  In contrast, the Clever Hans cheat is more severe as it causes in-distribution failure. Thirdly, the aforementioned empirical observations are specific to Transformers, and the theoretical arguments rely crucially on properties of the Transformer (such as its non-recurrence and convolution, or its self-attention modules). Our argument however only relies on the teacher-forcing objective with no reliance on the Transformer architecture, and is demonstrated even for the recurrent Mamba architecture. 
\end{remark}

\end{document}